\documentclass[11pt]{article}
\usepackage[a4paper,top=3cm,bottom=2cm,left=3cm,right=3cm,marginparwidth=1.75cm]{geometry}

\usepackage[normalem]{ulem}
\usepackage[none]{hyphenat}
\usepackage{breakcites}
\usepackage[utf8]{inputenc} 
\usepackage[T1]{fontenc}    
\usepackage{hyperref}       
\usepackage{url}            
\usepackage{booktabs}       
\usepackage{amsfonts}       
\usepackage{nicefrac}       
\usepackage{microtype}      
\usepackage{xcolor}         

\usepackage{float}
\usepackage{amsthm}
\usepackage{mathtools}
\usepackage{soul}
\mathtoolsset{showonlyrefs}
\usepackage{bm}
\usepackage{enumitem}
\usepackage{multirow}
\usepackage{nameref}
\usepackage{wrapfig}
\usepackage{scalerel}
\allowdisplaybreaks
\usepackage{dsfont}
\usepackage[noend]{algpseudocode}
\usepackage{xcolor}
\usepackage{thmtools,thm-restate}
\usepackage[linesnumbered,ruled,vlined]{algorithm2e}

\SetCommentSty{mycommfont}
\SetKwInput{KwInput}{Input}                
\SetKwInput{KwOutput}{Output}  
\SetKwInput{KwInitialization}{Initialization}

\newtheorem{definition}{Definition}[section]

\newtheorem{assumption}{Assumption}[section]

\newtheorem{lemma}{Lemma}[section]
\newtheorem{proposition}{Proposition}[section]
\newtheorem{property}{Property}[section]

\DeclareMathAlphabet{\mathbsf}{OT1}{cmss}{bx}{n}
\DeclareMathAlphabet{\mathssf}{OT1}{cmss}{m}{sl}

\DeclareMathOperator*{\argmin}{arg\,min}

\DeclarePairedDelimiterX{\infdivx}[2]{(}{)}{%
	#1\;\delimsize\|\;#2%
}

\newcommand{\dimOne}{k_1}	
\newcommand{\dimTwo}{k_2}	
\newcommand{\dimThree}{k_3}	
\newcommand{\rvv}{{\mathssf{v}}}	
\newcommand{\rvx}{{\mathssf{x}}}	
\newcommand{\rvz}{{\mathssf{z}}}	
\newcommand{\rvbv}{{\mathbsf{v}}} 
\newcommand{\rvbx}{{\mathbsf{x}}} 
\newcommand{\svbv}{{\mathbf{v}}} 
\newcommand{\svbx}{{\mathbf{x}}} 
\newcommand{\svby}{{\mathbf{y}}} 
\newcommand{\bV}{{\mathbf{V}}} 
\newcommand{\bM}{{\mathbf{M}}} 
\newcommand{\bN}{{\mathbf{N}}} 
\newcommand{\bU}{{\mathbf{U}}} 
\newcommand{\bH}{{\mathbf{H}}} 
\newcommand{\hbH}{\hat{\mathbf{H}}} 

\newcommand{\bg}{\boldsymbol{g}}
\newcommand{\br}{\boldsymbol{r}}
\newcommand{\bd}{\boldsymbol{d}}

\newcommand{\Reals}{\mathbb{R}} 

\newcommand{\bphi}{\boldsymbol{\phi}} 
\newcommand{\btheta}{\boldsymbol{\theta}} 
\newcommand{\ThetaStar}{{\Theta}^*} %
\newcommand{\Thetai}{{\Theta}^{(i)}}
\newcommand{\varPhii}{{\varPhi}^{(i)}}
\newcommand{\ThetaOne}{{\Theta}^{(1)}} 
 
\newcommand{\Thetal}{{\Theta}^{(\dimThree)}} 
\newcommand{\varPhiOne}{{\varPhi}^{(1)}} 
 
\newcommand{\PhiOne}{{\Phi}^{(1)}} 
\newcommand{\PhiTwo}{{\Phi}^{(2)}} 
\newcommand{\varPhil}{{\varPhi}^{(\dimThree)}} 
\newcommand{\ThetaStari}{{\Theta}^{*(i)}}
\newcommand{\ThetaStarOne}{{\Theta}^{*(1)}} 
\newcommand{\ThetaStarTwo}{{\Theta}^{*(2)}} 
 
\newcommand{\hTheta}{\hat{\Theta}} %
\newcommand{\hThetan}{\hat{\Theta}_n} %
\newcommand{\hThetaEps}{\hat{\Theta}_{\epsilon,n}}
\newcommand{\DensityExpfam}{f_{\rvbx}( \svbx;\btheta )} 
\newcommand{\DensityX}{f_{\rvbx}( \svbx;\Theta )} 
\newcommand{\DensityXfun}{f_{\rvbx}( \cdot;\Theta )} 
\newcommand{\DensityXTrue}{f_{\rvbx}( \svbx;\ThetaStar )} 
\newcommand{\DensityXTruefun}{f_{\rvbx}( \cdot;\ThetaStar )} 
\newcommand{\cX}{\mathcal{X}} 
\newcommand{\cR}{\mathcal{R}} 
\newcommand{\cB}{\mathcal{B}} 
\newcommand{\bcR}{\boldsymbol{\mathcal{R}}} 
\newcommand{\cU}{\mathcal{U}} 
\newcommand{\cL}{\mathcal{L}} 
\newcommand{\UniformI}{\cU_{\cX_i}}
\newcommand{\Uniform}{\cU_{\cX}}
\newcommand{\Uniformfun}{\cU_{\cX}(\cdot)}

\newcommand{\tphi}{\tilde{\phi}}
\newcommand{\parameterSet}{\Lambda} %
\newcommand{\DensityDifference}{f_{\rvbx}( \svbx;\ThetaStar - \Theta )}
\newcommand{\DensityDifferencefun}{f_{\rvbx}( \cdot;\ThetaStar - \Theta )}
\newcommand{\infdiv}{D\infdivx}
\newcommand{\lambdaMin}{\lambda_{\min}}
\newcommand{\phiMax}{\phi_{\max}}

\newcommand{\Expectation}{\mathbb{E}}

\newcommand{\vect}{\text{vec}}
\newcommand{\hH}{\hat{H}}
\newcommand{\Prob}{\mathbb{P}}

\newcommand{\epsTwo}{\epsilon_2}

\newcommand{\epsFour}{\epsilon_4}

\newcommand{\deltaTwo}{\delta_2}
\newcommand{\deltaThree}{\delta_3}
\newcommand{\deltaFour}{\delta_4}

\newcommand{\hthetan}{\hat{\theta}_n}

\let\svthefootnote\thefootnote
\newcommand\freefootnote[1]{%
	\let\thefootnote\relax%
	\footnotetext{#1}%
	\let\thefootnote\svthefootnote%
}

\title{A Computationally Efficient Method for Learning Exponential Family Distributions\freefootnote{Accepted for publication at the 35th Conference on Neural Information Processing Systems (NeurIPS 2021).}}

\author{%
	Abhin Shah
	\\
	MIT\\
	\texttt{abhin@mit.edu} \\
	\and
	Devavrat Shah \\
	MIT \\
	\texttt{devavrat@mit.edu}
	\and
	Gregory W. Wornell
	\\
	MIT \\
	\texttt{gww@mit.edu}
}
\date{}
\begin{document}
\sloppy
\maketitle
\begin{abstract}
 We consider the question of learning the natural parameters of a $k$-parameter \textit{minimal} 
 exponential family from i.i.d. samples in a computationally and statistically efficient manner. 
We focus on the setting where the support as well as the natural parameters are appropriately bounded.
While the traditional maximum likelihood estimator for this class of 
exponential family is consistent, asymptotically normal, and asymptotically efficient, evaluating it is computationally hard.
In this work, we propose a computationally efficient estimator that is consistent 
as well as asymptotically normal under mild conditions. 
We provide finite sample guarantees to achieve an 
($\ell_2$) error of $\alpha$ in the parameter estimation 
with sample complexity $O({\sf poly}(k/\alpha))$ 
and computational complexity  
${O}({\sf poly}(k/\alpha))$. 
To establish these results, we show that, at the population level, our method can be viewed 
as the maximum likelihood estimation of a re-parameterized distribution belonging to the same 
class of exponential family. 
Further, we show that our estimator can be interpreted as a solution to minimizing a particular Bregman score as well as an instance of minimizing the \textit{surrogate} likelihood. 
\end{abstract}
\section{Introduction}
\label{sec_intro}


We are interested in the problem of learning the natural parameters of a \textit{minimal} exponential family with bounded support.
Consider a $p$-dimensional random vector $\rvbx = (\rvx_1 , \cdots, \rvx_p)$ with support $\cX \subset \Reals^{p}$. 
An exponential family is a set of parametric probability distributions with probability densities of the following canonical form
\begin{align}
\DensityExpfam \propto \exp \big(\btheta^T \bphi(\svbx) + \beta(\svbx) \big), \label{eq:exp_fam}
\end{align} 
where $\svbx \in \cX$ is a realization of the underlying random variable $\rvbx$, $\btheta \in \Reals^{k}$ is the natural parameter, $\bphi : \cX \rightarrow \Reals^{k}$ is the natural statistic, $k$ denotes the number of parameters, and $\beta$ is the log base function. For representational convenience, we shall utilize the following equivalent
representation of \eqref{eq:exp_fam}:
\begin{align}
\DensityX \propto \exp\bigg(  \big\langle\big\langle \Theta, \Phi(\svbx) \big\rangle \big\rangle\bigg) = \exp\bigg(  \sum\nolimits_{i \in [k_1], j \in [k_2], l \in [k_3]} \Theta_{ijl} \times \Phi_{ijl}(\svbx) \bigg)  \label{eq:densityXfun}
\end{align}
where $\Theta = [\Theta_{ijl}] \in \Reals^{k_1\times k_2 \times k_3}$ is the natural parameter, $\Phi = [\Phi_{ijl}] : \cX \rightarrow \Reals^{k_1\times k_2 \times k_3}$ is the natural statistic, $k_1 \times k_2 \times k_3 - 1= k$, and $\big\langle\big\langle\Theta, \Phi(\svbx)  \big\rangle\big\rangle$ denotes the tensor inner product, i.e., the sum of product of entries of $\Theta$ and $\Phi(\svbx)$.
An exponential family is \textit{minimal} if there does not exist a nonzero tensor $\bU \in \Reals^{\dimOne\times  \dimTwo \times  \dimThree} $ 
such that $\big\langle\big\langle \bU, \Phi(\svbx) \big\rangle\big\rangle$ is equal to a constant for all $\svbx \in \cX$.


The notion of exponential family was first introduced by Fisher \cite{Fisher1934} and was later generalized by Darmois \cite{Darmois1935}, Koopman \cite{Koopman1936}, and Pitman \cite{Pitman1936}. Exponential families play an important role in statistical inference and arise in many diverse applications for a variety of reasons: (a) they are analytically tractable, (b) they arise as the solutions to several natural optimization problems on the space of probability distributions, (c) they have robust generalization property (see \cite{Brown1986, BarndorffNielsen2014} for details).

\textit{Truncated} (or bounded) exponential family, first introduced by Hogg and Craig \cite{HoggC1956}, is a set of parametric probability distributions resulting from truncating the support of an exponential family. \textit{Truncated} exponential families share the same parametric form with their non-truncated counterparts up to a normalizing constant. These distributions arise in many applications where we can observe only a truncated dataset (truncation is often imposed by during data acquisition) e.g., geolocation tracking data can only be observed up to the coverage of mobile signal, police department can often monitor crimes only within their city’s boundary.

The natural parameter $\Theta$ specifies a particular distribution in the exponential family. If the natural statistic $\Phi$ and the support of $\rvbx$ (i.e., $\cX$) are known, then learning a distribution in the exponential family is equivalent to learning the corresponding natural parameter $\Theta$.  
Despite having a long history, there has been limited progress on learning natural parameter $\Theta$ of a \textit{minimal truncated} exponential family. More precisely, there is no known method (without any abstract condition) that is both computationally and statistically efficient for learning natural parameter of the \textit{minimal truncated} exponential family considered in this work.

\subsection{Contributions}
\label{subsec_contributions}
As the primary contribution of this work, we provide a computationally tractable method with statistical guarantees for learning distributions in \textit{truncated minimal} exponential families.
Formally, the learning task of interest is estimating the true natural parameter $\ThetaStar$ from i.i.d. samples of $\rvbx$ obtained from $\DensityXTruefun$.
We focus on the setting where $\ThetaStar$ and $\Phi$ are appropriately bounded (see Section \ref{sec:prob_formulation}).
 We summarize our contributions in the following two categories.\\

%
%

\noindent{\bf 1.\ Computationally Tractable Estimator: Consistency, Normality, Finite Sample Guarantees.} 
Given $n$ samples $\svbx^{(1)} \cdots , \svbx^{(n)}$ of $\rvbx$, we propose the following novel loss function to learn a distribution belonging to the exponential family in \eqref{eq:densityXfun}:
\begin{align}
\cL_{n}(\Theta)  = \frac{1}{n} \sum_{t = 1}^{n} \exp\big( -\big\langle \big\langle \Theta, \varPhi(\svbx^{(t)}) \big\rangle \big\rangle \big), \label{eq:loss_function}
\end{align}
where $\varPhi(\cdot) = \Phi(\cdot) - \Expectation_{\Uniform} [\Phi(\cdot)]$ with $\Uniform$ being the uniform distribution over $\cX$. 
We establish that the estimator $\hThetan$ obtained by minimizing $\cL_{n}(\Theta)$ over all $\Theta$ in the constraint set $\parameterSet$, i.e., 
\begin{align}
\hThetan \in \argmin_{\Theta \in \parameterSet} \cL_{n}(\Theta), \label{eq:estimator}
\end{align}
is consistent and (under mild further restrictions) asymptotically normal (see Theorem \ref{thm:consistency_normality}). We obtain an $\epsilon$-optimal solution $\hThetaEps$ of the convex minimization problem in \eqref{eq:estimator} (i.e., $\cL_{n}(\hThetaEps) \leq \cL_{n}(\hThetan) + \epsilon$) by implementing a projected gradient descent algorithm with $O(\mathrm{poly}(k_1k_2/\epsilon))$\footnote[1]{We let $k_3 = O(1)$. See Section \ref{sec:prob_formulation}.} iterations (see Lemma \ref{lemma:gradient_descent}). Finally, we provide rigorous finite sample guarantees for $\hThetaEps$ (with $\epsilon = O(\alpha^2)$) to 
achieve an error of $\alpha$ (in the tensor $\ell_2$ norm) with respect to the true
natural parameter $\ThetaStar$ with $O(\mathrm{poly}(k_1k_2/\alpha))$ samples and $O(\mathrm{poly}(k_1k_2/\alpha))$ computations (see Theorem \ref{thm:finite_sample}). 
By letting certain additional structure on the natural parameter, we allow our framework to capture various constraints on the natural parameter including sparse, low-rank, sparse-plus-low-rank (see Section \ref{subsec:examples}).\\

\noindent{\bf 2.\ Connections to maximum likelihood estimation (MLE) of a re-parameterized distribution.}
We establish connections between our method and the MLE of the distribution $\DensityDifferencefun$. 
We show that the estimator that minimizes the population version of the loss function in \eqref{eq:loss_function} i.e.,
$$\cL(\Theta)  = \Expectation \Big[\exp\big( -\big\langle \big\langle \Theta, \varPhi(\rvbx) \big\rangle \big\rangle \big)\Big].$$
is equivalent to the estimator that minimizes the  Kullback-Leibler (KL) divergence between $\Uniform$ (the uniform distribution on $\cX$) and $\DensityDifferencefun$ (see Theorem \ref{theorem:GRISMe-KLD}). Therefore, at the population level, our method can be viewed as the MLE of the parametric family $\DensityDifferencefun$.
We show that the KL divergence (and therefore $\cL(\Theta)$) is minimized if and only if 
$\Theta = \ThetaStar$, and this connection provides an intuitively pleasing justification of the estimator in \eqref{eq:estimator}.

\subsection{Related Works}
\label{subsec_related_work}
In this section, we look at the related works on learning exponential family. Broadly speaking, there are two line of approaches to overcome the computational hardness of the MLE : (a) approximating the MLE and (b) selecting a surrogate objective. Given the richness of both of approaches, we cannot do justice in providing a full overview. Instead, we look at a few examples from both. Next, we look at some of the related works that focus on learning a class of exponential family. More specifically, we look at works on (a) learning the Gaussian distribution and (b) learning exponential family Markov random fields (MRFs). 
Finally, we explore some works on the powerful technique of score matching.
In Appendix \ref{appendix:related_works}, we further review works on learning exponential family MRFs,
score-based methods (including the related literature on Stein discrepancy) and latent variable graphical models (since these capture sparse-plus-low-rank constraints on the parameters similar to our framework).\\

\noindent{\bf Approximating the MLE.} Most of the techniques falling in this category approximate the MLE by approximating the log-partition function. A few examples include : (a) approximating the gradient of log-likelihood with a stochastic estimator by minimizing the contrastive divergence \cite{Hinton2002}; (b) upper bounding the log-partition function by an iterative tree-reweighted belief propagation algorithm \cite{WainwrightJW2003}; (c) using Monte Carlo methods like importance sampling for estimating the partition function \cite{RobertC2013}. Since these methods approximate the partition function, they come at the cost of an approximation error or result in a biased estimator. \\

\noindent{\bf Selecting surrogate objective.} This line of approach selects an easier-to-compute surrogate objective that completely avoids the partition function. A few examples are as follows : (a) pseudo-likelihood estimators \cite{Besag1975} approximate the joint distribution with the product of conditional distributions, each of which only represents the distribution of a single variable conditioned on the remaining variables; 
(b) score matching \cite{HyvarinenD2005, Hyvarinen2007} minimizes the Fisher divergence between the true log density and the model log density. Even though score matching does not require evaluating the partition function, it is computationally expensive as it requires computing third order derivatives for optimization; (c) kernel Stein discrepancy \cite{LiuLJ2016, ChwialkowskiSG2016} measures the kernel mean discrepancy between a data distribution and a model density using the Stein's identity. This measure is directly characterized by the choice of the kernel and there is no clear objective for choosing the right kernel \cite{WenliangSSG2019}.\\
%

\noindent{\bf Learning the Gaussian distribution.}
Learning the Gaussian distribution is a special case of learning exponential family distributions. There has been a long history of learning Gaussian distributions in the form of learning Gaussian graphical models e.g. the neighborhood selection scheme \cite{MeinshausenB2006}, the graphical lasso \cite{FriedmanHT2008}, the CLIME \cite{CaiLL2011}, etc. However, finite sample analysis of these methods require various hard-to-verify conditions e.g. the restricted eigenvalue condition, the incoherence assumption (\cite{WainwrightRL2006, JalaliRVS2011}), bounded eigenvalues of the precision matrix, etc.
A recent work \cite{KelnerKMM2019} provided an algorithm whose sample complexity, for a specific subclass of Gaussian graphical models, match the information-theoretic lower bound of \cite{WangWR2010} without the aforementioned hard-to-verify conditions.\\

\noindent{\bf Learning Exponential Family Markov Random Fields (MRFs).}
MRFs can be naturally represented as exponential family distributions via the principle of maximum entropy (see \cite{WainwrightJ2008}). A popular method for learning  MRFs is estimating node-neighborhoods (fitting conditional distributions of each node conditioned on the rest of the nodes) because the natural parameter is assumed to be node-wise- sparse. A recent line of work has considered a subclass of node-wise-sparse pairwise continuous MRFs where the node-conditional distribution of $\rvx_i \in \cX_i$ for every $i$ arise from an exponential family as follows:
\begin{align}
f_{\rvx_i | \rvx_{-i}}(x_i | \svbx_{-i} = x_{-i}) \propto \exp \big( \big[  \theta_i  + \sum_{j \in [p], j \neq i} \theta_{ij}  \phi(x_j)  \big] \phi(x_i) \big),  \label{eq:conditional}
\end{align}
where $\phi(x_i)$ is the natural statistics and $\theta_i  + \sum_{j \in [p], j \neq i} \theta_{ij}  \phi(x_j) $ is the natural parameter.\footnote[2]{Under node-wise-sparsity, $\sum_{j \in [p], j \neq i} |\theta_{ij}|$ is bounded by a constant for every $i \in [p]$.} Yang et al. \cite{YangRAL2015} showed that only the following joint distribution is consistent with the node-conditional distributions in \eqref{eq:conditional} :
\begin{align}
f_{\rvbx}(\svbx) \propto \exp \big( \sum_{i \in [p]} \theta_i  \phi(x_i) + \sum_{j \neq i} \theta_{ij}  \phi(x_i) \phi(x_j) \big).  \label{eq:joint}
\end{align}
To learn the node-conditional distribution in \eqref{eq:conditional} for linear $\phi(\cdot)$ (i.e., $\phi(x) = x$), Yang et al. \cite{YangRAL2015} proposed an $\ell_1$ regularized node-conditional log-likelihood. However, their
finite sample analysis 
required the following conditions: incoherence, dependency (see \cite{WainwrightRL2006, JalaliRVS2011}), bounded moments of the variables, and local smoothness of the log-partition function. Tansey et al. \cite{TanseyPSR2015} extended the approach in \cite{YangRAL2015} to vector-space MRFs (i.e., vector natural parameters and natural statistics) and non-linear $\phi(\cdot)$. They proposed a sparse group lasso (see \cite{ SimonFHT2013}) regularized node-conditional log-likelihood and an alternating direction method of multipliers based approach to solving the resulting optimization problem. However, their analysis required same conditions as \cite{YangRAL2015}. 

While node-conditional log-likelihood has been a natural choice for learning exponential family MRFs, M-estimation \cite{VuffrayMLC2016, VuffrayML2019, ShahSW2021} and maximum pseudo-likelihood estimator \cite{NingZL2017, YangNL2018, DaganDDA2021} have recently gained popularity. The objective function in M-estimation is a sample average and the estimator is generally consistent and asymptotically normal. 
Shah et al. \cite{ShahSW2021} proposed the following M-estimation (inspired from \cite{VuffrayMLC2016, VuffrayML2019}) for vector-space MRFs and non-linear $\phi(\cdot)$: with $\UniformI$ being the uniform distribution on $\cX_i$ and $\tphi(x_i) = \phi(x_i) - \int_{x_i'} \phi(x_i') \UniformI(x_i') dx_i'$
\begin{align}
\arg \min \frac{1}{n} \sum_{i = 1}^{n} \exp\Big(- \big[\theta_i  \tphi(x_i) + \sum_{j \in [p], j \neq i} \theta_{ij}  \tphi(x_i) \tphi(x_j) \big] \Big). \label{eq:shah}
\end{align}
They provided an entropic descent algorithm (borrowing from \cite{VuffrayML2019}) to solve the optimization in \eqref{eq:shah} and their finite-sample bounds rely on bounded domain of the variables and a condition (naturally satisfied by linear $\phi(\cdot)$) that lower bounds the variance of a non-constant random variable. 

Yuan et al. \cite{YuanLZLL2016} considered a broader class of sparse pairwise exponential family MRFs compared to \cite{YangRAL2015}. They studied the following joint distribution with natural statistics $\phi(\cdot)$ and $\psi(\cdot)$
\begin{align}
f_{\rvbx}(\svbx) \propto \exp \Big( \sum_{i \in [p]} \theta_i  \phi(x_i) + \sum_{j \neq i} \theta_{ij}  \psi(x_i, x_j) \Big).  \label{eq:yuan}
\end{align}
They proposed an $\ell_{2,1}$ regularized joint likelihood and an $\ell_{2,1}$ regularized node-conditional likelihood. They also presented a Monte-Carlo approximation to these estimators via proximal gradient descent. Their finite-sample analysis required restricted strong convexity (of the Hessian of the negative log-likelihood of the joint density) and bounded moment-generating function of the variables. 

	Building upon \cite{VuffrayML2019} and \cite{ShahSW2021}, Ren et al. \cite{RenMVL2021} addressed learning continuous exponential family distributions through a series of numerical experiments. They considered unbounded distributions and allowed for terms corresponding to multi-wise interactions in the joint density. However, they considered only monomial natural statistics.
	Further, they 
	assume node-wise-sparsity of the parameters as in MRFs and their
	 estimator is defined as a series of node-wise optimization problems. 
	
	In summary, tremendous progress has been made on learning the sub-classes of exponential family in \eqref{eq:joint} and \eqref{eq:yuan}. However, this sub-classes are restricted by the assumption that the natural parameters are node-wise-sparse. For example, none of the existing methods for exponential family MRFs work in the setting where the natural parameters have a low-rank constraint.\\
\noindent{\bf Score-based method.}
	A scoring rule $S(\svbx, Q)$ is a numerical score assigned to a realization $\svbx$ of a random variable $\rvbx$ and it measures the quality of a predictive distribution $Q$ (with probability density $q(\cdot)$). 
	If $P$ is the true distribution of $\rvbx$, the divergence $D(P,Q)$ associated with a scoring rule is defined as $\Expectation_{P}[S(\rvbx, Q) - S(\rvbx, P)]$. The MLE is an example of a scoring rule with $S(\cdot, Q) = - \log q(\cdot)$ and the resulting divergence is the KL-divergence.
	
	To bypass the intractability of MLE, \cite{HyvarinenD2005} proposed an alternative scoring rule with $S(\cdot, Q) = \Delta \log q(\cdot) + \frac{1}{2} \| \nabla \log q(\cdot)\|^2_2$ where $\Delta$ is the Laplacian operator, $\nabla$ is the gradient and $\| \cdot \|_2$ is the $\ell_2$ norm. This method is called \textit{score matching} and the resulting divergence is the Fisher divergence. Score matching is widely used for estimating unnormalizable probability distributions because computing the scoring rule $S(\cdot, Q)$ does not require knowing the partition function. Despite the flexibility of this approach, it is computationally expensive in high dimensions since it requires computing the trace of the unnormalized density's Hessian (and its derivatives for optimization). Additionally, it breaks down for models in which the second derivative grows very rapidly.
	
	In \cite{LiuKW2019}, the authors considered estimating truncated exponential family using the principle of {score matching}. They build on the framework of generalized score matching \cite{Hyvarinen2007} and proposed a novel estimator that minimizes a weighted Fisher divergence. They showed that their estimator is a special case of minimizing a Stein Discrepancy. However, their finite sample analysis relies on certain hard-to-verify assumptions, for example, the assumption that the optimal parameter is well-separated from other neighboring parameters in terms of their population objective. Further, their estimator lacks the useful properties of asymptotic normality and asymptotic efficiency.

\subsection{Useful notations and outline}
\label{subsec_notations}
\noindent{\bf Notations.}
For any positive integer $t$, let $[t] \coloneqq \{1,\cdots, t\}$.
For a deterministic sequence $v_1, \cdots , v_t$, we let $\svbv \coloneqq (v_1, \cdots, v_t)$. 
For a random sequence $\rvv_1, \cdots , \rvv_t$, we let $\rvbv \coloneqq (\rvv_1, \cdots, \rvv_t)$. For a matrix $\bM \in \Reals^{u \times v}$, we denote the element in $i^{th}$ row and $j^{th}$ column by $M_{ij}$, the singular values of the matrix by $\sigma_i(\bM)$ for $i \in [\min\{u, v\}]$, the matrix maximum norm by $\|\bM\|_{\max} \coloneqq \max_{i \in [u], j \in [v]} |M_{ij}|$, the entry-wise $L_{1,1}$ norm by $\|\bM\|_{1,1} \coloneqq \sum_{i \in [u], j \in [v]} |M_{ij}|$, the nuclear norm by $\|\bM\|_{\star} \coloneqq \sum_{i \in [\min\{u, v\}]} \sigma_i(\bM)$.
We denote the Frobenius or Trace inner product of matrices $\bM, \bN \in \Reals^{u \times v}$ by  $\langle \bM, \bN \rangle \coloneqq \sum_{i \in [u], j \in [v]} M_{ij} N_{ij}$. 
For a matrix $\bM \in \Reals^{u \times v}$, we denote a generic norm on $\Reals^{u \times v}$ by $\cR(\bM)$ and denote the associated dual norm by $\cR^*(\bM) \coloneqq \sup \{\langle \bM, \bN \rangle | \cR(\bN) \leq 1\}$ where $\bN \in \Reals^{u \times v}$. 
For a tensor $\bU \in \Reals^{u \times v \times w}$, we denote its $(i,j,l)$ entry by $U_{ijl}$, its $l^{th}$ slice (obtained by fixing the last index) by $U_{::l}$ or $U^{(l)}$, the tensor maximum norm (with a slight abuse of notation) by  $\|\bU\|_{\max} \coloneqq \max_{i \in [u], j \in [v], l \in [w]} |U_{ijl}|$, and the tensor norm by $\| \bU\|_{\mathrm{T}} \coloneqq \sqrt{\sum_{i \in [u], j \in [v], l \in [w]} U^2_{ijl}}$. 
We denote the  tensor inner product of tensors $\bU, \bV \in \Reals^{u\times v \times w}$ by  $\langle \langle \bU, \bV \rangle \rangle \coloneqq \sum_{i \in [u], j \in [v], l \in [w]} U_{ijl} V_{ijl}$. 
We denote the vectorization of the tensor $\bU \in \Reals^{u \times v \times w}$ by $\vect(\bU) \in \Reals^{uvw \times 1}$ (the ordering of the elements is not important as long as it is consistent).
Let $\boldsymbol{0} \in \Reals^{\dimOne \times \dimTwo \times \dimThree}$ denote the tensor with every entry zero. We denote a $p$-dimensional ball of radius $b$ centered at $0$ by $\cB(0,b)$.\\

\noindent{\bf Outline.}
In Section \ref{sec:prob_formulation}, we formulate the problem of interest, state our assumptions, and provide examples. In Section \ref{sec:algorithm}, we provide our loss function and algorithm. In Section \ref{sec:main results}, we present our main results including the connections to the MLE of $\DensityDifferencefun$, 
consistency, asymptotic normality, and finite sample guarantees. 
In Section \ref{sec:misc}, we conclude, provide some remarks, discuss limitations as well as some directions for future work.
See supplementary for organization of the Appendix.

\section{Problem Formulation}
\label{sec:prob_formulation}
Let $\rvbx = (\rvx_1 , \cdots, \rvx_p)$ be a $p-$dimensional vector of continuous random variables.\footnote[3]{Even though we focus on continuous variables, our framework applies equally to discrete variables.} For any $i \in [p]$, let the support of $\rvx_i$ be $\cX_i \subset \Reals$.
Define $\cX \coloneqq \prod_{i=1}^p \cX_i$.
Let $\svbx = (x_1, \cdots, x_p) \in \cX$ be a realization of $\rvbx$. In this work, we assume that the random vector $\rvbx$ belongs to an exponential family with bounded support (i.e., length of $\cX_i$ is bounded) along with certain additional constraints. More specifically, we make certain assumptions 
on the natural parameter $\Theta \in \Reals^{\dimOne \times  \dimTwo \times  \dimThree}$, and on the natural statistic $\Phi(\svbx) : \cX \rightarrow \Reals^{\dimOne \times  \dimTwo \times  \dimThree}$ as follows.\\

 
 \noindent{\bf Natural parameter $\Theta$.}
We focus on natural parameters with bounded norms. However, instead of having such constraints on the natural parameter $\Theta$ as it is, we decompose  $\Theta$ into $\dimThree$ slices (or matrices) and have slice specific constraints.
The key motivation for this is to broaden the class of exponential family covered by our formulation.
For example, this decomposability allows our formulation to en-capture the sparse-plus-low-rank decomposition of $\Theta$ in addition to only sparse or only low-rank decompositions of $\Theta$ (see Section \ref{subsec:examples}). This is precisely the reason for considering tensor natural parameters instead of matrix natural parameters. Further, we assume $\dimThree = O(1)$ i.e., it does not scale with $p$. We formally state this assumption below.
\begin{assumption}\label{bounds_parameter}
	(Bounded norms of $\Theta$.)
	For every $i \in [\dimThree]$, we let $\cR_i(\Thetai) \leq r_i$ where $\Thetai \in \Reals^{\dimOne\times  \dimTwo}$ is the $i^{th}$ slice of $\Theta$, $\cR_i : \Reals^{\dimOne\times  \dimTwo} \rightarrow \Reals_+$ is a norm and $r_i$ is a known constant. 
	This decomposition is represented compactly by $\bcR(\Theta) \leq \br$ where $\bcR(\Theta) = (\cR_1(\ThetaOne) , \cdots, \cR_{\dimThree}(\Thetal) )$ and $\br= (r_1,\cdots,r_{\dimThree})$. 
\end{assumption}
We define $\parameterSet$ to be the set of all natural parameters satisfying Assumption \ref{bounds_parameter} i.e., $\parameterSet \coloneqq \{\Theta : \bcR(\Theta) \leq \br\}$. For any $\tilde{\Theta}, \bar{\Theta} \in \parameterSet$ and $t \in [0,1]$, we have $\bcR(t \tilde{\Theta} + (1-t) \bar{\Theta} ) \leq t \bcR(\tilde{\Theta}) + (1-t) \bcR(\bar{\Theta} ) \leq t\br + (1-t)\br = \br$. Therefore, $t \tilde{\Theta} + (1-t) \bar{\Theta}  \in \parameterSet$ and the constraint set $\parameterSet$ is a convex set.\\

 \noindent{\bf Natural Statistic $\Phi$.}
For mathematical simplicity, we center the natural statistic $\Phi(\cdot)$ such that their integral with respect to the uniform density on $\cX$ (i.e., $\Uniform$) is zero. $\Uniform$ is well-defined because the support $\cX$ is a strict subset of $\Reals^{p}$ i.e., $\cX \subset \Reals^{p}$.
\begin{definition}\label{def:css}
	(Centered natural statistics). 
The centered natural statistics are defined as follows: 
	\begin{align}
	\varPhi(\cdot) & \coloneqq \Phi(\cdot) - \Expectation_{\Uniform} [\Phi(\rvbx)]. \label{eq:CSS}
	\end{align}
\end{definition}

In this work, we focus on bounded natural statistics which may enforce certain restrictions on the length of support $\cX$. See Section \ref{subsec:examples} for examples. We define two notions of boundedness.  First, we make the following assumption to be able to bound the tensor inner product between the natural parameter $\Theta$ and the centered natural statistic $\varPhi(\cdot)$ (see Appendix \ref{appendix:frobenius_bound}).
\begin{assumption}\label{bounds_statistics}
	(Bounded dual norms of $\varPhi$). For every $i \in [\dimThree]$ and norm $\cR_i$, we  assume that the dual norm $\cR^*_i$ of the $i^{th}$ slice of the centered natural statistic i.e., $\varPhii$ is bounded by a constant $d_i$. Formally, for any $i \in [\dimThree]$ and $\svbx \in \cX$, $\cR^*_i(\varPhii(\svbx)) \leq d_i$. This is represented compactly by $\bcR^*(\varPhi(\svbx)) \leq \bd$ where $\bcR^*(\varPhi(\svbx)) = (\cR^*_1(\varPhiOne(\svbx)) , \cdots, \cR^*_{\dimThree}(\varPhil(\svbx)) )$ and $\bd = (d_1,\cdots,d_{\dimThree})$. 
\end{assumption}

Next, we assume that the tensor maximum norm of the centered natural statistic $\varPhi(\cdot)$ is bounded by a constant $\phiMax$. This assumption is stated formally below.
\begin{assumption}\label{bounds_statistics_maximum}
	(Bounded tensor maximum norm of $\varPhi$). For any $\svbx \in \cX$, $\|\varPhi(\svbx)\|_{\max} \leq \phiMax$.
\end{assumption}

%



\noindent{\bf  The Exponential Family.}
Summarizing, $\rvbx$ belongs to a \textit{minimal truncated} exponential family with probability density function as follows
\begin{align}
\DensityX \propto \exp\Big(  \Big\langle \Big\langle \Theta, \Phi(\svbx) \Big\rangle \Big\rangle \Big). \label{eq:densityX}
\end{align}
where the natural parameter $\Theta \in \Reals^{\dimOne\times  \dimTwo \times \dimThree}$ is such that $\bcR(\Theta) \leq \br$ and the natural statistic $\Phi(\svbx) : \cX \rightarrow \Reals^{\dimOne\times  \dimTwo \times \dimThree}$ is such that for any $\svbx \in \cX$, $\bcR^*(\varPhi(\svbx)) \leq \bd$ and $\|\varPhi(\svbx)\|_{\max} \leq \phiMax$.

Let $\ThetaStar$ denote the true natural parameter of interest and $\DensityXTrue$ denote the true distribution of $\rvbx$. Naturally, we assume $\bcR(\ThetaStar) \leq \br$. Formally, the learning task of interest is as follows:

{\bf Goal.}
(Natural Parameter Recovery). Given $n$ independent samples of $\rvbx$ i.e., $\svbx^{(1)} \cdots , \svbx^{(n)}$ obtained from $\DensityXTrue$, compute an estimate $\hTheta$ of $\ThetaStar$ in polynomial time such that $\|\ThetaStar - \hTheta\|_{\mathrm{T}}$ is small.

\subsection{Examples}
\label{subsec:examples}
We will first present examples of natural parameters that satisfy Assumption \ref{bounds_parameter}. Next, we will present examples of natural statistics along with the corresponding support that satisfy Assumptions \ref{bounds_statistics}, and \ref{bounds_statistics_maximum}. See Appendix \ref{appendix:computational_cost} and \ref{appendix:examples} for more discussion on these examples.\\

\noindent{\bf Examples of natural parameter.}
We provide examples in Table \ref{table:1} to illustrate the decomposability of $\Theta$ as in Assumption \ref{bounds_parameter}. We will revisit these examples briefly in Section \ref{sec:main results} and in-depth in Appendix \ref{appendix:computational_cost}. Assumption \ref{bounds_parameter} should be viewed as a potential flexibility in the problem specification i.e., a practitioner has the option to choose from a variety of constraints on the natural parameters (that could be handled by our framework). For example, in some real-world applications the parameters are sparse while in some other real-world applications the parameters have a low-rank and a practitioner could choose either depending on the application at hand.
\begin{table}[!htbp]
	\centering
	\caption{A few examples of natural parameter $\Theta$.} 
	\begin{tabular}{lll}
		\toprule
		Decomposition & $\dimThree$ &  Convex Relaxation \\
		\midrule
		\midrule
		Sparse decomposition ($\ThetaStar = (\ThetaStarOne)$) & $1$ & $\|\ThetaStarOne\|_{1,1} \leq r_1$\\ 
		\midrule
		Low-rank decomposition ($\ThetaStar = (\ThetaStarOne)$) & $1$ & $\|\ThetaStarOne\|_{\star} \leq r_1$\\
		\midrule
		Sparse-plus-low-rank decomposition & $2$ & $\|\ThetaStarOne\|_{1,1} \leq r_1$ and $\|\ThetaStarTwo\|_{\star} \leq r_2$\\
		($\ThetaStar = (\ThetaStarOne, \ThetaStarTwo)$) &~ \\
		\bottomrule
	\end{tabular}
	\label{table:1}
	\vspace{-.1in}
\end{table}
For the sparse-plus-low-rank decomposition, it is more natural to think about the \textit{minimality} of the exponential family in terms of matrices as opposed to tensors. See Appendix \ref{appendix:examples} for details.\\


\noindent{\bf Examples of natural statistic.}
The following are a few example of natural statistics (along with the corresponding support) that fall in-line with Assumptions \ref{bounds_statistics} and  \ref{bounds_statistics_maximum}. 
\begin{enumerate}[leftmargin=*,topsep=-0pt,itemsep=-0pt]
	\item \textit{Polynomial statistics}: Suppose the natural statistics are polynomials of $\rvbx$ with maximum degree $l$, i.e., $\prod_{i \in [p]} x_i^{l_i}$ such that $l_i \geq 0$ $\forall i \in [p]$ and $\sum_{i \in [p]} l_i \leq l$. If $\cX = [0,b]$ for $b \in \Reals$, then $\phiMax = 2b^l$. If $\ThetaStar$ has a sparse decomposition 
and $\cX = [0,b]$ for $b \in \Reals$, then $\bcR^*(\varPhi(\svbx)) \leq 2b^k$. Further, if $\ThetaStar$ has a  low-rank decomposition, $l = 2$, and $\cX = \cB(0,b)$ for $b \in \Reals$, then $\bcR^*(\varPhi(\svbx)) \leq 2(1+b^2)$. Finally, if $\ThetaStar$ has a sparse-plus-low-rank decomposition, $l = 2$, and $\cX = \cB(0,b)$ for $b \in \Reals$, then $\bcR^*(\varPhi(\svbx)) \leq (2b^2, 2+2b^2)$.
	\item \textit{Trigonometric statistics}: Suppose the natural statistics are sines and cosines of $\rvbx$ with $l$ different frequencies, i.e., $\sin(\sum_{i \in [p]}l_ix_i)$ $\cup$ $\cos(\sum_{i \in [p]}l_ix_i)$ such that $l_i \in [l] \cup \{0\}$. For any $\cX \subset \Reals^p$, $\phiMax = 2$. If $\ThetaStar$ has a sparse decomposition, then $\bcR^*(\varPhi(\svbx)) \leq 2$ for any $\cX \subset \Reals^p$.
\end{enumerate}
Our framework also allows combinations of polynomial and trigonometric statistics (see Appendix \ref{appendix:examples}).\footnote[4]{We believe that for polynomial and/or trigonometric natural statistics, Assumptions \ref{bounds_statistics} and  \ref{bounds_statistics_maximum} would hold whenever the domain of $\cX$ is appropriately bounded.}
\section{Algorithm}
\label{sec:algorithm} 
We propose a novel, computationally tractable loss function
drawing inspiration from the recent advancements in exponential family Markov Random Fields \cite{VuffrayMLC2016, VuffrayML2019, ShahSW2021}.\\

\noindent{\bf The loss function and the estimator.}
The loss function, defined below, is 
an empirical average of the inverse of the function of $\rvbx$ that the probability density $\DensityX$ is proportional to (see \eqref{eq:densityX}). 

\begin{definition}[The loss function] Given $n$ samples $\svbx^{(1)} \cdots , \svbx^{(n)}$ of $\rvbx$, the loss function maps $\Theta \in \Reals^{\dimOne\times  \dimTwo \times \dimThree}$ to $\cL_{n}(\Theta) \in \Reals$ defined as 
	\begin{align}
	\cL_{n}(\Theta)  = \frac{1}{n} \sum_{t = 1}^{n} \exp\big( -\big\langle \big\langle \Theta, \varPhi(\svbx^{(t)}) \big\rangle\big\rangle \big). \label{eq:sampleGISMe}
	\end{align}
\end{definition}

The proposed estimator $\hThetan$ produces an estimate of $\ThetaStar$ by minimizing the loss function $\cL_{n}(\Theta)$ over all natural parameters $\Theta$ satisfying Assumption \ref{bounds_parameter} i.e.,
\begin{align}
\hThetan \in \argmin_{\Theta \in \parameterSet} \cL_{n}(\Theta).
 \label{eq:GRISMe}
\end{align}

For any $\epsilon > 0$, $\hThetaEps$ is an $\epsilon$-optimal solution of $\hThetan$ if $\cL_{n}(\hThetaEps) \leq \cL_{n}(\hThetan) + \epsilon$.
The optimization in \eqref{eq:GRISMe} is a convex minimization problem (i.e., minimizing a convex function $\cL_{n}$ over a convex set $\parameterSet$) and has efficient implementations for finding an $\epsilon$-optimal solution. 
Although alternative algorithms 
(including Frank-Wolfe) can be used, we provide a projected gradient descent algorithm below. \\


	\begin{algorithm}[H]
		\SetCustomAlgoRuledWidth{0.4\textwidth} 
		\KwInput{ $\eta,  \tau, \parameterSet$}
		\KwOutput{$\hThetaEps$}
		\KwInitialization{$\Theta_{(0)} = \boldsymbol{0}$}
		{
			\For{$t = 0,\cdots,\tau$}
			{
				$\Theta_{(t+1)} \leftarrow \argmin_{\Theta \in \parameterSet} \| \Theta_{(t)} - \eta  \nabla \cL_{n}(\Theta_{(t)}) - \Theta\|_{\mathrm{T}}$
			}
		$\hThetaEps \leftarrow \Theta_{(\tau+1)}$
		}
		\caption{Projected Gradient Descent}
		\label{alg:GradientDescent}
	\end{algorithm}

The following Lemma shows that running sufficient iterations of the projected gradient descent in Algorithm \ref{alg:GradientDescent} results in an  $\epsilon$-optimal solution of $\hThetan$.  

\begin{restatable}{lemma}{lemmaGD}\label{lemma:gradient_descent}
	Let Assumptions \ref{bounds_parameter}, \ref{bounds_statistics} and \ref{bounds_statistics_maximum} be satisfied. Let $\eta = 1/\dimOne \dimTwo\dimThree\phiMax^2\exp(\br^T \bd)$. Then, Algorithm \ref{alg:GradientDescent} returns an $\epsilon$-optimal solution $\hThetaEps$ as long as
		\begin{align}
		\tau & \geq \frac{2\dimOne \dimTwo \dimThree\phiMax^2\exp(\br^T\bd)}{\epsilon} \| \hTheta_n\|^2_{\mathrm{T}}. \label{eq:tau}
		\end{align}
	Further, ignoring the dependence on $\dimThree$, $\phiMax$, $\br$ and $\bd$, $\tau$ in \eqref{eq:tau} scales as $O\big(\mathrm{poly}\big( \frac{\dimOne \dimTwo}{\epsilon}\big)\big)$.
\end{restatable}

The proof of Lemma \ref{lemma:gradient_descent} can be found in Appendix \ref{appendix:algorithm_section_proofs}. The proof outline is as follows : (a) First, we prove the smoothness property of $\cL_{n}(\Theta)$. (b) Next, we complete the proof using a standard result from convex optimization for the projected gradient descent algorithm for smooth functions.

\section{Analysis and Main results}
\label{sec:main results}
In this section, we provide our analysis and main results. First, we focus 
on the connection between our method and the MLE of $\DensityDifferencefun$.
Then, we establish consistency and asymptotic normality of our estimator. Finally, we provide non-asymptotic finite sample guarantees to recover $\ThetaStar$.\\
\noindent{\bf 1. Connection with MLE of $\DensityDifferencefun$.}
First, we will establish a connection between the population version of the loss function in \eqref{eq:sampleGISMe} (denoted by $\cL(\Theta)$) and the KL-divergence of the uniform density on $\cX$ with respect to $\DensityDifference$. Then, using \textit{minimality} of the exponential family, we will show that this KL-divergence and $\cL(\Theta)$ are minimized if and only if $\Theta = \ThetaStar$. This provides a justification for the estimator in \eqref{eq:GRISMe} as well as helps us obtain consistency and asymptotic normality of $\hTheta_n$.

For any $\Theta \in \parameterSet$, $\cL(\Theta)  = \Expectation \Big[\exp\big( -\big\langle \big\langle \Theta, \varPhi(\rvbx) \big\rangle \big\rangle \big)\Big].$
The following result shows that the population version of the estimator in \eqref{eq:GRISMe} is equivalent to the maximum likelihood estimator of $\DensityDifference$. 
\begin{restatable}{theorem}{theoremKLD}\label{theorem:GRISMe-KLD}
	With $\infdiv{\cdot}{\cdot}$ representing the KL-divergence, 
	\begin{align}\label{eq:thm.1}
	\argmin_{\Theta \in \parameterSet} \cL(\Theta) =
	\argmin_{\Theta \in \parameterSet}
	\infdiv{\Uniformfun}{\DensityDifferencefun}. 
	\end{align}
	Further, the true parameter $\ThetaStar$ is the unique minimizer of $\cL(\Theta)$.
\end{restatable}

The proof of Theorem \ref{theorem:GRISMe-KLD} can be found in Appendix \ref{appendix:proof of theorem:GRISE-KLD}. The proof outline is as follows : (a) First, we express $\DensityDifferencefun$ in terms of $\cL(\Theta)$ (b) Next, we complete the proof by simplifying the KL-divergence between $\Uniformfun$ and $\DensityDifferencefun$. \\
\noindent{\bf 2. Consistency and Normality.}
We establish consistency and asymptotic normality of the proposed estimator $\hTheta_n$ by invoking the asymptotic theory of M-estimation. We emphasize that, from Theorem \ref{theorem:GRISMe-KLD}, the population version of $\hTheta_n$ is equivalent to the maximum likelihood estimate of $\DensityDifferencefun$ and not $\DensityXfun$. Moreover, there is no clear connection between $\hTheta_n$ and the finite sample maximum likelihood estimate of $\DensityXfun$ or $\DensityDifferencefun$. Therefore, we cannot invoke the asymptotic theory of MLE to show consistency and asymptotic normality of $\hTheta_n$.

Let $A(\ThetaStar)$ denote the covariance matrix of $\vect\big(\varPhi(\rvbx)\exp\big( -\big\langle \big\langle \ThetaStar, \varPhi(\rvbx) \big\rangle \big\rangle \big)\big)$. Let $B(\ThetaStar)$ denote the cross-covariance matrix of $\vect(\varPhi(\rvbx))$ and $\vect(\varPhi(\rvbx) \exp\big( -\big\langle \big\langle \ThetaStar, \varPhi(\rvbx) \big\rangle \big\rangle \big))$. Let ${\cal N}(\bm{\mu}, \bm{\Sigma})$ represent the multi-variate Gaussian distribution with mean vector $\bm{\mu}$ and covariance matrix $\bm{\Sigma}$.
\begin{restatable}{theorem}{theoremconsistencynormality}
		\label{thm:consistency_normality}
		Let Assumptions \ref{bounds_parameter}, \ref{bounds_statistics}, and \ref{bounds_statistics_maximum} be satisfied.
		Let $\hThetan$ be a solution of \eqref{eq:GRISMe}. Then, as $n\to \infty$, $\hThetan \stackrel{p}{\to} \ThetaStar$. Further, assuming  $\ThetaStar \in \text{interior}(\parameterSet)$ and $B(\ThetaStar)$ is invertible, we have
		$\sqrt{n} \times \vect( \hThetan - \ThetaStar ) \stackrel{d}{\to} {\cal N}(\vect(\boldsymbol{0}),B(\ThetaStar)^{-1} A(\ThetaStar) B(\ThetaStar)^{-1})$.
\end{restatable}
The proof of Theorem \ref{thm:consistency_normality} can be found in Appendix \ref{appendix:proof of thm:consistency_normality}. The proof is based on two key observations : (a) $\hThetan$ is an $M$-estimator and (b) $\cL(\Theta)$ is uniquely minimized at $\ThetaStar$.\\

\noindent{\bf 3. Finite Sample Guarantees.}
To provide the non-asymptotic guarantees for recovering $\ThetaStar$, we require the following assumption on the smallest eigenvalue of the autocorrelation matrix of $\vect(\varPhi(\rvbx))$.
\begin{assumption}\label{lambdamin}(Positive eigenvalue of the autocorrelation matrix of $\varPhi$.)
	Let $\lambdaMin$ denote the minimum eigenvalue of $\Expectation_{\rvbx}[\vect(\varPhi(\rvbx)) \vect(\varPhi(\rvbx))^T]$. We assume $\lambdaMin$ is strictly positive i.e., $\lambdaMin > 0$.
\end{assumption}

We also make use of the following property of the matrix norms.
\begin{property}
	\label{property:norms}
	For any norm $\tilde{\cR} : \Reals^{\dimOne \times \dimTwo} \rightarrow \Reals_+$, and matrix $\bM \in \Reals^{\dimOne \times  \dimTwo}$, there exists $g$ such that $\tilde{\cR}(\bM) \leq g  \dimOne   \dimTwo \|\bM\|_{\max}$. 
\end{property}
For most matrix norms of interest including entry-wise $L_{p,q}$ norm $(p,q \geq 1)$, Schatten $p$-norm $(p \geq 1)$, and operator $p-$norm $(p \geq 1)$, we have $g = 1$ as shown in Appendix \ref{appendix:dual_norm}.  

Let $\bg = (g_1 ,\cdots, g_{\dimThree})$ where $~\forall i \in [\dimThree], g_i$ is such that $\cR^*_i(\bM) \leq g_i  \dimOne   \dimTwo \|\bM\|_{\max}$ with $\cR^*_i$ being the dual norms from Assumption \ref{bounds_statistics}.

Theorem \ref{thm:finite_sample} below shows that, with enough samples, the $\epsilon$-optimal solution of $\hTheta_n$ 
is close to the true natural parameter in the tensor norm with high probability. 
\begin{restatable}{theorem}{theoremfinite}\label{thm:finite_sample}
	Let $\hThetaEps$ be an
	$\epsilon$-optimal solution of $\hTheta_n$ obtained from Algorithm \ref{alg:GradientDescent} for $\epsilon$ of the order $O(\alpha^2 \lambdaMin)$. Let Assumptions \ref{bounds_parameter}, \ref{bounds_statistics}, \ref{bounds_statistics_maximum}, and \ref{lambdamin} be satisfied. Recall Property \ref{property:norms}. Then, for any $\delta \in (0,1)$, we have $\|\hThetaEps - \ThetaStar\|_{\mathrm{T}} \leq \alpha$
	with probability at least $1-\delta$ as long as
		\begin{align}
		n & \geq O\bigg(\frac{ \dimOne^2  \dimTwo^2 }{\alpha^4 \lambdaMin^2}\log\Big(\frac{\dimOne \dimTwo}{\delta}\Big)\bigg). 
\label{eq:sample_complexity}
	\end{align}
	The computational cost scales as $O\Big(\frac{\dimOne \dimTwo}{\alpha^2}\max\big(\dimOne \dimTwo n, c(\parameterSet)\big)\Big)$
	where $c(\parameterSet)$ is the cost of projection onto $\parameterSet$. Further, ignoring the dependence on $\delta$, $\lambdaMin$, and $c(\parameterSet)$, $n$ in \eqref{eq:sample_complexity} (as well as the associated computational cost) scales as $O\big(\mathrm{poly}\big( \frac{\dimOne \dimTwo}{\alpha}\big)\big)$.	
\end{restatable}
The proof of Theorem \ref{thm:finite_sample} can be found in Appendix \ref{appendix_proof_finite_sample}. The proof is based on two key properties of the loss function $\cL_n(\Theta)$ : (a) with enough samples, the loss function $\cL_n(\Theta)$ naturally obeys the restricted strong convexity with high probability and (b) with enough samples, $\| \nabla \cL_n(\ThetaStar) \|_{\max}$ is bounded with high probability. See the proof for the dependence of the sample complexity and the computational complexity on $\dimThree, \br, \bd, \bg$ and $\phiMax$. 

The computational cost of projection onto $\parameterSet$ i.e., $c(\parameterSet)$ is typically polynomial in $\dimOne \dimTwo$. In Appendix \ref{appendix:computational_cost}, we provide the computational cost for the example constraints on the natural parameter $\Theta$ from Section \ref{subsec:examples} i.e., 
sparse decomposition, low-rank decomposition, and sparse-plus-low-rank decomposition.\\

\noindent{\bf 4. Comparison with the traditional MLE.}
To contextualize our method, we compare it with the MLE of the parametric family $\DensityXfun$. The MLE of $\DensityXfun$ minimizes the following loss function 
\begin{align}
\min - \frac{1}{n} \sum_{t = 1}^{n} \big\langle \big\langle  \Theta, \Phi(\svbx^{(t)}) \big\rangle \big\rangle + \log \int_{\svbx \in \cX} \exp\big(\big\langle \big\langle \Theta, \Phi(\svbx \big\rangle\big\rangle\big) d\svbx. \label{eq:mle}
\end{align}
The maximum likelihood estimator has many attractive asymptotic properties 
: (a) consistency (see \cite[Theorem~17]{Ferguson2017}), i.e., as the sample size goes to infinity, the bias in the estimated parameters goes to zero, (b) asymptotic normality (see \cite[Theorem~18]{Ferguson2017}), i.e., as the sample size goes to infinity, normalized estimation error coverges to a Gaussian distribution
and (c) asymptotic efficiency (see \cite[Theorem~20]{Ferguson2017}), i.e., as the sample size goes to infinity, the variance 
in the estimation error
attains the minimum possible value among all consistent estimators.
Despite having these useful asymptotic properties of consistency, normality, and efficiency, 
computing the maximum likelihood estimator is computationally hard \cite{Valiant1979, JerrumS1989}.

Our method can be viewed as a computationally efficient proxy for the MLE. More precisely, our method is computationally tractable as opposed to the MLE while retaining the useful properties of consistency and asymptotic normality. However, our method misses out on asymptotic efficiency. This raises an important question for future work --- \emph{can computational and asymptotic efficiency be achieved by a single estimator for this class of exponential family?}
\section{Conclusion, Remarks, Limitations, Future Work}
\label{sec:misc}
In this section, we conclude, provide a few remarks, discuss the limitations of our work as well as some interesting future directions.\\

\noindent{\bf Conclusion.}
In this work, we provide a computationally and statistically efficient method to learn distributions in a \textit{minimal truncated} $k$-parameter exponential family from i.i.d. samples. We propose a novel estimator via minimizing a convex loss function and obtain consistency and asymptotic normality of the same. We provide rigorous finite sample analysis to achieve an $\alpha$-approximation to the true natural parameters with $O(\mathrm{poly}(k/\alpha))$ samples and $O(\mathrm{poly}(k/\alpha))$ computations. We also provide an interpretation of our estimator in terms of a maximum likelihood estimation.\\

\noindent{\bf Node-wise-sparse exponential family MRFs vs general exponential family.}  We highlight that the focus of our work is beyond the exponential families associated with node-wise-sparse MRFs and towards general exponential families. The former focuses on  local assumptions on the parameters such as node-wise-sparsity and the sample complexity depends logarithmically on the parameter dimension i.e., $O(\mathrm{log}(k))$.  In contrast, our work can handle global structures on the parameters (e.g., a low-rank constraint) and there are no prior work that can handle such global structures with sample complexity $O(\mathrm{log}(k))$. Similarly, for node-wise-sparse MRFs there has been a lot of work to relax the assumptions required for learning (see the discussion on Assumption \ref{lambdamin} below). Since our work focuses on global structures associated with the parameters, we leave the question of relaxing the assumptions required for learning as an open question. Likewise, the interaction screening objective \cite{VuffrayMLC2016} and generalized interaction screening objective \cite{VuffrayML2019, ShahSW2021} were designed for node-wise parameter estimation i.e., they require the parameters to be node-wise-sparse and are less useful when the parameters have a global structure. On the contrary, our loss function is designed to accommodate global structures on the parameters.\\
	%
	%
	%

	\noindent{\bf Assumption \ref{lambdamin}.} For node-wise-sparse pairwise exponential family MRFs (e.g., Ising models), which is a special case of the setting considered in our work, Assumption \ref{lambdamin} is proven (e.g., Appendix T.1 of \cite{ShahSW2021} provides one such analysis for a condition that is equivalent to Assumption \ref{lambdamin} for sparse continuous graphical model). However, such analysis typically requires (a) a bound on the infinity norm of the parameters and a bound on the degree of each node or (b) a bound on the $\ell_1$ norm of the parameters associated with each node. Since the focus of our work is beyond the exponential families associated with node-wise-sparse MRFs, we view Assumption \ref{lambdamin} as an adequate condition to rule out certain singular distributions (as evident in the proof of Proposition \ref{prop:rsc_GISMe} where this condition is used to effectively lower bounds the variance of a non-constant random variable) and expect it to hold for most real-world applications. Further, we highlight that the MLE in \eqref{eq:mle} remains computationally intractable even under Assumption \ref{lambdamin}. To see this, one could again focus on node-wise-sparse pairwise exponential family MRFs where Assumption \ref{lambdamin} is proven and the MLE is still known to be computationally intractable.\\
	
	\noindent{\bf Sample Complexity.} We do not assume $p$ (the dimension of $\rvbx$) to be a constant and think of $\dimOne$ and $\dimTwo$ as implicit functions of $p$. Typically, for an exponential family, the quantity of interest is the number of parameters i.e., $k$ and this quantity scales polynomially in $p$ e.g., $k = O(p^2)$ for Ising model, $k = O(p^t)$ for t-wise MRFs over binary alphabets. Therefore, in this scenario, the dependence of the sample complexity on $p$ would also be $O(\mathrm{poly}(p))$. Further, the $1/\alpha^4$ dependence of the sample complexity seems fundamental to our loss function. For learning node-wise-sparse MRFs, this dependence is in-line with some prior works that use a similar loss function \cite{ShahSW2021, VuffrayML2019} as well as that do not use a similar loss function \cite{KlivansM2017}. While it is known that for learning node-wise-sparse MRFs \cite{VuffrayMLC2016} and truncated Gaussian \cite{DaskalakisGTZ2018} one could achieve a better dependence of $1/\alpha^2$, it is not yet clear how the lower bound on the sample complexity would depend on $\alpha$ for the general class of exponential families considered in this work (which may not be sparse or Gaussian). \\
	
	\noindent{\bf Practicality of Algorithm \ref{alg:GradientDescent}.} 
	While the optimization associated with Algorithm \ref{alg:GradientDescent} is a convex minimization problem (i.e., \eqref{eq:GRISMe}) and the computational complexity of Algorithm \ref{alg:GradientDescent} is polynomial in the parameter dimension and the error tolerance, computing the gradient of the loss function requires centering of the natural statistics (see \eqref{eq:gradient-GISMe}). If the natural statistics are polynomials or trigonometric, centering them should be relatively straightforward (since the integrals would have closed-form expressions). In other cases, centering them may not be polynomial-time and one might require an assumption of computationally efficient sampling or that obtaining approximately random samples of $\rvbx$ is computationally efficient \cite{DiakonikolasKSS2021}.\\

\noindent{\bf  Limitations and Future Work.}
First, in our current framework, we assume boundedness of the support. While, conceptually, most non-compact distributions could be truncated by introducing a controlled amount of error, we believe this assumption could be lifted as for exponential families: $\mathbb{P}(|\rvx_i| \geq \delta  \log \gamma) \leq c\gamma^{-\delta}$ where $c>0$ is a constant and $\gamma > 0$. Alternatively, the notion of multiplicative regularizing distribution from \cite{RenMVL2021} could also be used. 
Second, while the population version of our estimator has a nice interpretation in terms of maximum likelihood estimation, the finite sample version of our estimator does not have a similar interpretation.
We believe there could be connections with the Bregman score and this is an important direction for immediate future work.
Third, while our estimator is computationally efficient, consistent, and asymptotically normal, it is not asymptotically efficient. Investigating the possibility of a single estimator that achieves computational and asymptotic efficiency for this class of exponential family could be an interesting future direction. Lastly, building on our framework, empirical study is an important direction for future work.

\section*{Acknowledgements}
	This work was supported, in part, by NSR under Grant No.\ CCF-1816209, ONR under Grant No. N00014-19-1-2665, the NSF TRIPODS Phase II grant towards Foundations of Data Science Institute, the MIT-IBM project on time series anomaly detection, and the KACST project on Towards Foundations of Reinforcement Learning.

\bibliographystyle{abbrv}
{\small
	\bibliography{Mybib_papers}
	
}
\clearpage
\appendix
\section*{Appendix}
\textbf{Organization.}
In Appendix \ref{appendix:related_works}, we provide additional discussion on exponential family Markov random fields, score-based methods, 
as well as review the related literature on Stein discrepancy and latent variable graphical models.
In Appendix \ref{appendix:algorithm_section_proofs}, we state and prove the smoothness property of the loss function as well as provide the proof of Lemma \ref{lemma:gradient_descent}. 
In Appendix \ref{appendix:proof of theorem:GRISE-KLD}, we provide the proof of Theorem \ref{theorem:GRISMe-KLD}. 
In Appendix \ref{appendix:proof of thm:consistency_normality}, we provide the proof of Theorem \ref{thm:consistency_normality}. 
In Appendix \ref{appendix:rsc}, we provide the restricted strong convexity property of the loss function.
In Appendix \ref{appendix:bounds on the gradient of the GISMe}, we provide bounds on the tensor maximum norm of the gradient of the loss function evaluated at the true natural parameter. 
In Appendix \ref{appendix_proof_finite_sample}, we provide the proof of Theorem \ref{thm:finite_sample}. 
In Appendix \ref{appendix:computational_cost}, we provide the computational cost for the example constraints on the natural parameter $\Theta$. 
In Appendix \ref{appendix:examples}, we provide a discussion on the examples of natural parameter and natural statistics from Section \ref{subsec:examples}. 
In Appendix \ref{appendix:dual_norm}, we provide a discussion on Property \ref{property:norms}.\\

\noindent{\bf Additional Notations.} We denote the $\ell_p$ norm $(p \geq 1)$ of a vector $\svbv \in \Reals^t$ by $\| \svbv\|_{p} \coloneqq (\sum_{i=1}^{t}|v_i|^p)^{1/p}$ and its $\ell_{\infty}$ norm by $\|\svbv\|_{\infty} \coloneqq \max_{i \in [t]} |v_i|$. For a matrix $\bM \in \Reals^{u \times v}$, we denote the spectral norm by $\|\bM\| \coloneqq \max_{i \in [\min\{u, v\}]} \sigma_i(\bM)$ and the Frobenius norm by $\| \bM\|_{\mathrm{F}} \coloneqq \sqrt{\sum_{i \in [u], j \in [v]} M^2_{ij}}$. 
For a tensor $\bU \in \Reals^{u \times v \times w}$, we let $\|\bU\|_{1,1,1} \coloneqq \sum_{i \in [u], j \in [v], l \in [w]} |U_{ijl}|$.

\section{Related Works}
\label{appendix:related_works}
In this Section, we review additional works on exponential family Markov random fields, score-based methods, 
as well as the related literature on Stein discrepancy and latent variable graphical models.
\subsection{Exponential Family Markov Random Fields}
\label{subsubsec_exp_fam_mrf_rl_app}
Having reviewed some of the works on sparse exponential family MRFs in Section \ref{subsec_related_work}, we present here a brief overview of a few other works on the same.

Following the lines of \cite{YangRAL2015}, the authors in \cite{SuggalaKR2017} proposed an $\ell_1$ regularized node-conditional log-likelihood to learn the node-conditional density in \eqref{eq:conditional} for non-linear $\phi(\cdot)$. They used an alternating minimization technique and proximal gradient descent to solve the resulting optimization problem. However, their analysis required restricted strong convexity, bounded domain of the variables, non-negative node parameters, and hard-to-verify assumptions on gradient of the population loss. 

In \cite{YangNL2018}, the authors introduced a non-parametric component to the node-conditional density in \eqref{eq:conditional}  while focusing on linear $\phi(\cdot)$. More specifically, they focused on the following joint density: 
\begin{align}
f_{\rvbx}(\svbx) \propto \exp \Big( \sum_{i \in [p]}  \eta_i(x_i) + \sum_{j \neq i} \theta_{ij}  x_i x_j \Big),  \label{eq:yang}
\end{align}
where $\eta_i(\cdot)$ is the non-parametric node-wise term. They proposed a node-conditional pseudo-likelihood (introduced in \cite{NingZL2017}) regularized by a non-convex penalty and an adaptive multi-stage convex relaxation method to solve the resulting optimization problem. However, their finite-sample bounds require bounded moments of the variables, sparse eigenvalue condition on their loss function, and local smoothness of the log-partition function.	
In \cite{SunKX2015}, the authors investigated infinite dimensional sparse pairwise exponential family MRFs where they assumed that the node and edge potentials lie in a Reproducing Kernel Hilbert space  (RKHS). They used a penalized version of the score matching objective of \cite{HyvarinenD2005}. However, their finite-sample analysis required incoherence and dependency conditions (see \cite{WainwrightRL2006, JalaliRVS2011}). In \cite{LinDS2016}, the authors considered the joint distribution in \eqref{eq:yuan} restricting the variables to be non-negative. They proposed a group lasso regularized generalized score matching objective \cite{Hyvarinen2007} which is a generalization of the score matching objective \cite{HyvarinenD2005} to non-negative data. However, their finite-sample analysis required the incoherence condition.

\subsection{Score-based and Stein discrepancy methods}
\label{subsubsec_score_based_methods_rl_app}
Having mentioned the principle behind and an example for the score-based method in Section \ref{subsec_related_work}, we briefly review a few other score-based methods in relation to the Stein discrepancy.

Stein discrepancy is a quantitative measure of how well a predictive density $q(\cdot)$ fits the density of interest $p(\cdot)$ based on the classical Stein's identity. Stein's identity defines an infinite number of identities indexed by a critic function $f$ and does not require evaluation of the partition function like the {score matching} method. By focusing on Stein discrepancy constructed from a RKHS, the authors in \cite{LiuLJ2016} and \cite{ChwialkowskiSG2016} independently proposed the kernel Stein discrepancy as a test statistic to access the goodness-of-fit for unnormalized densities. The authors in \cite{LiuLJ2016}  and \cite{BarpBDGM2019} showed that the Fisher divergence, which was the minimization criterion used by the {score matching} method, can be viewed a special case of the kernel Stein discrepancy with a specific, fixed critic function $f$. In \cite{BarpBDGM2019}, the authors showed that a few other methods (including the contrastive divergence by \cite{Hinton2002}) can also be viewed as a kernel Stein discrepancy with respect to a different class of critics. Despite the kernel Stein discrepancy being a natural criterion for fitting computationally hard models, there is no clear objective for choosing the right kernel and the kernels typically chosen (e.g. \cite{SunKX2015, StrathmannSLSG2015, SriperumbudurFGHK2017, SutherlandSAG2018} ) are insufficient for complex datasets as pointed out by \cite{WenliangSSG2019}.

In \cite{DaiLDHGSS2019}, the authors exploited the primal-dual view of the MLE to avoid estimating the normalizing constant at the price of introducing dual variables to be jointly estimated. They showed that many other methods including the contrastive divergence by \cite{Hinton2002}, pseudo-likelihood by \cite{Besag1975}, score matching by \cite{HyvarinenD2005} and minimum Stein discrepancy estimator by \cite{LiuLJ2016}, \cite{ChwialkowskiSG2016}, and \cite{BarpBDGM2019} are special cases of their estimator. However, this method results in expensive optimization problems since they rely on adversarial optimization (see \cite{RhodesXG2020} for details). {In \cite{LiuKJC2019}, the authors proposed an inference method for unnormalized models known as discriminative likelihood estimator. This estimator follows the KL divergence minimization criterion and is implemented via density ratio estimation and a Stein operator. However, this method requires certain hard-to-verify conditions.}

\subsection{Literature on Latent Variable Graphical Models}
In recent years, sparse-plus-low-rank matrix recovery has received considerable attention in machine learning and statistical inference, e.g., robust PCA \cite{CandesLMW2011}, latent variable graphical models \cite{ChandrasekaranPW2010}. 
Latent variable graphical models has a variety of applications including assessing the functional interactions between neurons recorded from two brain areas \cite{VinciVSK2018, NaKK2019}. In latent variable graphical models, there are variables not present in observations. The presence of such variables leads to a challenge in learning the graphical model. The graphical model corresponding to the conditional distribution of the observed variables conditioned on the latent variables is in general different from the graphical model corresponding to the marginal distribution of the observed variables. The marginal graphical model consists of dependencies that are induced due to marginalization over the latent variables and typically consists of many more edges than the conditional graphical model. In \cite{ChandrasekaranPW2010}, authors considered latent variable Gaussian graphical models and exploited the observation that the precision matrix of the marginal graphical model can be decomposed into the superposition of a sparse matrix and a low-rank matrix. They provided a tractable convex program based on regularized maximum-likelihood to estimate the precision matrix. While the authors in \cite{ChandrasekaranPW2010} focused on simultaneous model selection consistency of both the sparse and low-rank components, the authors in \cite{MengEH2014} focused on estimating the precision matrix of latent variable Gaussian graphical model. They consider a regularized MLE estimator and utilize the \textit{almost strong convexity} \cite{KakadeSST2010} of the log-likelihood to derive non-asymptotic error bounds under the restricted Fisher eigenvalue and Structural Fisher Incoherence assumptions. Compared to \cite{MengEH2014}, our tensor norm error bounds are derived under mild condition. Additionally, our framework captures various constraints on the natural parameters in addition to the sparse-plus-low-rank constraint.
%

\section{Smoothness of the loss function and proof of Lemma \ref{lemma:gradient_descent}}
\label{appendix:algorithm_section_proofs}
In this Section, we will prove the smoothness of $\cL_{n}(\cdot)$ as well as prove Lemma \ref{lemma:gradient_descent}. However, before either of this, we provide bounds on the absolute tensor inner product between $\Theta$ and $\varPhi$ i.e., $\big|\big\langle \big\langle \Theta, \varPhi(\svbx) \big\rangle \big\rangle\big|$ for $\Theta \in \parameterSet$ and $\svbx \in \cX$. 

\subsection{Bounds on the absolute tensor inner product between $\Theta$ and $\varPhi$.}
\label{appendix:frobenius_bound}
We have
\begin{align}
\big|\big\langle \big\langle \Theta, \varPhi(\svbx) \big\rangle \big\rangle\big| \stackrel{(a)}{=}  \big| \sum_{i =1}^{\dimThree} \big\langle \Theta^{(i)}, \varPhii(\svbx) \big\rangle\big|  \stackrel{(b)}{\leq}  \sum_{i =1}^{\dimThree}\big|  \big\langle \Theta^{(i)}, \varPhii(\svbx) \big\rangle\big|  & \stackrel{(c)}{\leq} \sum_{i =1}^{\dimThree} \cR_i (\Theta^{(i)}) \times \cR^*_i(\varPhii(\svbx)) \\
& \stackrel{(d)}{\leq} \br^T \bd, \label{eq:bound_inner_product}
\end{align}
where $(a)$ follows from the definitions of a slice of a tensor, tensor inner product, and Frobenius inner product,
$(b)$ follows from the triangle inequality, $(c)$ follows from the definition of a dual norm, and $(d)$ follows from Assumptions \ref{bounds_parameter} and \ref{bounds_statistics}.

\subsection{Smoothness of the loss function}
\label{appendix:smoothness}
Now, we will state and prove our result for smoothness of $\cL_{n}(\Theta)$.
\begin{proposition}\label{proposition:smoothness}
	Under Assumptions \ref{bounds_parameter}, \ref{bounds_statistics} and \ref{bounds_statistics_maximum}, $\cL_{n}(\Theta)$ is a $\dimOne \dimTwo \dimThree \phiMax^2\exp(\br^T\bd )$ smooth function of $\Theta$.
\end{proposition}
\begin{proof}[Proof of Proposition \ref{proposition:smoothness}]
	To show $\dimOne \dimTwo  \dimThree\phiMax^2\exp(\br^T\bd )$ smoothness of $\cL_{n}(\Theta)$,
	we will show that the largest eigenvalue of the Hessian\footnote[5]{Ideally, one would consider the Hessian of $\cL_{n}(\vect({\Theta}))$. However, for the ease of the exposition we abuse the terminology.} of $\cL_{n}(\Theta)$ is upper bounded by $\dimOne \dimTwo \dimThree \phiMax^2\exp(\br^T\bd )$.
	
	First, we simplify the Hessian of $\cL_{n}(\Theta)$ i.e., $\nabla^2 \cL_{n}(\Theta)$. 
	The component of the Hessian of $\cL_{n}(\Theta)$ corresponding to $\Theta_{u_1v_1w_1}$
	and $\Theta_{u_2v_2w_2}$ for $u_1,u_2 \in [\dimOne]$, $v_1,v_2 \in [ \dimTwo]$  and $w_1,w_2 \in [ \dimThree]$ is given by
	\begin{align}
	\frac{\partial^2 \cL_{n}(\Theta)}{\partial \Theta_{u_1v_1w_1} \partial \Theta_{u_2v_2w_2}}   = \frac{1}{n} \sum_{t = 1}^{n} \varPhi_{u_1v_1w_1}(\svbx^{(t)}) \varPhi_{u_2v_2w_2}(\svbx^{(t)}) \exp\Big( -\Big\langle \Big\langle  \Theta, \varPhi(\svbx^{(t)}) \Big\rangle \Big\rangle \Big).  \label{eq:hessian-GISMe}
	\end{align}
	From the Gershgorin circle theorem, we know that the largest eigenvalue of any matrix is upper bounded by the largest absolute row sum or column sum. Let $\lambda_{\max}(\nabla^2 \cL_{n}(\Theta))$ denote the largest eigenvalue of $\nabla^2 \cL_{n}(\Theta)$. We have the following
	\begin{align}
	\lambda_{\max}(\nabla^2 \cL_{n}(\Theta)) \leq \max_{u_2,v_2,w_2} \sum_{u_1,v_1,w_1} \Big| \frac{\partial^2 \cL_{n}(\Theta)}{\partial \Theta_{u_1v_1w_1} \partial \Theta_{u_2v_2w_2}}  \Big|  & \stackrel{(a)}{\leq} \max_{u_2,v_2,w_2} \sum_{u_1,v_1,w_1} \phiMax^2  \exp(\br^T\bd)  \\ 
	& \leq \dimOne \dimTwo \dimThree \phiMax^2\exp(\br^T\bd ),
	\end{align}
	where $(a)$ follows from \eqref{eq:hessian-GISMe}, \eqref{eq:bound_inner_product}, and Assumption \ref{bounds_statistics_maximum}. Therefore, $\cL_{n}(\Theta)$ is a $\dimOne \dimTwo \dimThree \phiMax^2\exp(\br^T\bd )$ smooth function of $\Theta$.
\end{proof}
\subsection{Proof of Lemma \ref{lemma:gradient_descent}}
\label{appendix:proof_lemma_grad_des}
Next, we restate the Lemma \ref{lemma:gradient_descent} and provide the proof.
\lemmaGD*
\begin{proof}[Proof of Lemma \ref{lemma:gradient_descent}]
	Let us recall Theorem 10.6 from \cite{MeghanaN2016}.\\
	
	\cite[Theorem 10.6]{MeghanaN2016}:
	Let $L$ be a $c$-smooth convex function of a parameter vector $\theta \in \parameterSet$. Consider the following constrained optimization problem
	\begin{align}
	\min_{\theta \in \parameterSet} L(\theta).  \label{eq:constrained_original}
	\end{align}
	Let $\theta^*$ be an optimal solution of \eqref{eq:constrained_original}. Let $\theta^{(1)}, \cdots, \theta^{(t)}$ denote the iterates of the projected gradient descent algorithm with step size $\eta = 1/c$. Let $\theta^{(0)}$ denote the initialization of $\theta$ in the projected gradient descent algorithm. Then,
	\begin{align}
	L(\theta^{(t)}) - L(\theta^*) \leq \frac{2c}{t} \| \theta^{(0)} - \theta^*\|^2_2. \label{eq:constrained}
	\end{align}
	
	We will make direct use of this theorem in our proof. From Proposition \ref{proposition:smoothness}, $\cL_{n}(\Theta)$ is $c_1 \coloneqq  \dimOne \dimTwo\dimThree \phiMax^2\exp(\br^T\bd )$ smooth. Using \eqref{eq:constrained}, we have
	\begin{align}
	\cL_{n}(\Theta_{(\tau)}) - \cL_{n}(\hThetan) \leq \frac{2c_1}{\tau} \| \Theta_{(0)}- \hTheta_n\|^2_{\mathrm{T}}.
	\end{align}
	Plugging in $c_1 =  \dimOne \dimTwo \dimThree \phiMax^2\exp(\br^T\bd )$, $\tau = \dfrac{2\dimOne \dimTwo \dimThree \phiMax^2\exp(\br^T\bd)}{\epsilon} \| \hTheta_n\|^2_{\mathrm{T}}$, and $\Theta_{(0)} = \boldsymbol{0}$ we have
	\begin{align}
	\cL_{n}(\Theta_{(\tau)}) - \cL_{n}(\hThetan) \leq \epsilon.
	\end{align}
	Therefore, $\Theta_{(\tau)}$ is an $\epsilon$-optimal solution.
	
	We will now upper bound $\| \hTheta_n\|^2_{\mathrm{T}}$. First let us upper bound this tensor norm in terms of tensor maximum norm and therefore the matrix maximum norms. We have
	\begin{align}
	\| \hTheta_n\|^2_{\mathrm{T}} \leq \dimOne \dimTwo \dimThree \| \hTheta_n\|^2_{\max} = \dimOne \dimTwo \dimThree \max_{i \in [\dimThree]} \| \hTheta^{(i)}_n\|^2_{\max}.
	\end{align}
	Now, observe that  most matrix norms of interest including the entry-wise $L_{p,q}$ norm $(p,q \geq 1)$, the Schatten $p$-norm $(p \geq 1)$, and the operator $p$-norm $(p \geq 1)$ are bounded from below by the matrix maximum norm i.e., the matrix maximum norm is upper bounded if either of these matrix norms are upper bounded. Suppose $\forall i \in [\dimThree]$, $\cR_i$ is either the entry-wise $L_{p,q}$ norm $(p,q \geq 1)$, the Schatten $p$-norm $(p \geq 1)$, or the operator $p$-norm $(p \geq 1)$. Then, $\forall i \in [\dimThree]$, $\| \hTheta^{(i)}_n\|_{\max}  \leq \cR_i(\hTheta^{(i)}_n)$.  We have $\cR_i(\hTheta^{(i)}_n) \leq r_i $ from Assumption \ref{bounds_parameter} because $\hTheta^{(i)}_n \in \parameterSet$. Therefore, we have
	\begin{align}
	\| \hTheta_n\|^2_{\mathrm{T}} \leq \dimOne \dimTwo \dimThree \max_{i \in [\dimThree]} r_i^2.
	\end{align}
	Summarizing and using the fact that $\phiMax, \br, \bd, \dimThree$ are $O(1)$, we have
	\begin{align}
	\frac{2\dimOne \dimTwo \dimThree\phiMax^2\exp(\br^T\bd)}{\epsilon} \| \hTheta_n\|^2_{\mathrm{T}} \leq \frac{2\dimOne^2 \dimTwo^2 \dimThree^2\phiMax^2\exp(\br^T\bd)}{\epsilon}\max_{i \in [\dimThree]} r_i^2
	= O\bigg( \frac{ \dimOne^2  \dimTwo^2  }{\epsilon}\bigg).
	\end{align}
	
\end{proof}

\section{Proof of Theorem 4.1}
\label{appendix:proof of theorem:GRISE-KLD}
In this Section, we prove Theorem \ref{theorem:GRISMe-KLD}. We restate the Theorem below and then provide the proof. 

\theoremKLD*
\begin{proof}[Proof of Theorem \ref{theorem:GRISMe-KLD}]
	We will first express $\DensityDifferencefun$ in terms of $\cL(\Theta)$. We have
	\begin{align}
	\DensityDifference & = 
	\frac{\exp\big(  \big\langle  \big\langle \ThetaStar - \Theta, \Phi(\svbx) \big\rangle \big\rangle \big)}{\int_{\svby \in \cX} \exp\big(  \big\langle  \big\langle \ThetaStar - \Theta, \Phi(\svby) \big\rangle \big\rangle \big) d\svby} 
	\stackrel{(a)}{=}  \frac{\exp\big(  \big\langle  \big\langle \ThetaStar - \Theta, \varPhi(\svbx) \big\rangle \big\rangle \big)}{\int_{\svby \in \cX} \exp\big(  \big\langle  \big\langle \ThetaStar - \Theta, \varPhi(\svby) \big\rangle \big\rangle \big) d\svby} \\
	& \stackrel{(b)}{=} \frac{\DensityXTrue \exp\big(  -\big\langle \big\langle \Theta, \varPhi(\svbx) \big\rangle \big\rangle\big) }{\int_{\svby \in \cX}   \DensityXTrue \exp\big(  -\big\langle \big\langle \Theta, \varPhi(\svby) \big\rangle \big\rangle\big) d\svby}\\
	& \stackrel{(c)}{=} \frac{\DensityXTrue \exp\big(  -\big\langle \big\langle \Theta, \varPhi(\svbx) \big\rangle \big\rangle\big) }{\cL(\Theta)},  \label{eq:re-express difference density}
	\end{align}
	where $(a)$ follows because $\Expectation_{\Uniform} [\Phi(\rvbx)]$ is a constant, $(b)$ follows by dividing the numerator and the denominator by the constant $\int_{\svby \in \cX} \exp\big(  \big\langle \big\langle  \ThetaStar, \varPhi(\svby) \big\rangle \big\rangle \big) d\svby$ and using the definition of $\DensityXTrue$, and
	$(c)$ follows from definition of $\cL(\Theta)$. We will now simplify the KL-divergence between $\Uniformfun$ and $\DensityDifferencefun$. 
	\begin{align}
	\infdiv{\Uniformfun}{\DensityDifferencefun}  &  \stackrel{(a)}{=}  \Expectation_{\Uniform} \bigg[ \log\bigg( \dfrac{\Uniform(\cdot) \cL(\Theta) }{\DensityXTruefun \exp\big(  -\big\langle \big\langle \Theta, \varPhi(\cdot) \big\rangle \big\rangle \big) }\bigg) \bigg] \\
	&  \stackrel{(b)}{=} \Expectation_{\Uniform} \bigg[ \log\bigg( \dfrac{\Uniform(\cdot)  }{\DensityXTruefun  }\bigg) \bigg] + \Expectation_{\Uniform} \Big[ \Big\langle \Big\langle  \Theta, \varPhi(\cdot) \Big\rangle \Big\rangle \Big] +  \log \cL(\Theta) \\
	&  \stackrel{(c)}{=}  \Expectation_{\Uniform} \bigg[ \log\bigg( \dfrac{\Uniform(\cdot)  }{\DensityXTruefun  }\bigg) \bigg] +
	\Big\langle \Big\langle  \Theta, \Expectation_{\Uniform} [ \varPhi(\cdot) ] \Big\rangle \Big\rangle +  \log \cL(\Theta) \\
	&  \stackrel{(d)}{=} \Expectation_{\Uniform} \bigg[ \log\bigg( \dfrac{\Uniform(\cdot)  }{\DensityXTruefun  }\bigg) \bigg] + \log \cL(\Theta),
	\end{align}
	where $(a)$ follows from \eqref{eq:re-express difference density} and the definition of KL-divergence, $(b)$ follows because $\log(abc) = \log a + \log b + \log c$ and $\cL(\Theta)$ is a constant, $(c)$ follows from the linearity of the expectation and $(d)$ follows because $\Expectation_{\Uniform} [ \varPhi(\rvbx) ] = 0$ from Definition \ref{def:css}.
	Observing that the first term in the above equation is not dependent on $\Theta$, we can write
	\begin{align}
	\argmin_{\Theta \in \parameterSet}
	\infdiv{\Uniformfun}{\DensityDifferencefun} 
	= \argmin_{\Theta \in \parameterSet} \log \cL(\Theta)
	\stackrel{(a)}{=} \argmin_{\Theta \in \parameterSet}  \cL(\Theta),
	\end{align}
	where $(a)$ follows because $\log$ is a monotonic function. Further, the KL-divergence between $\Uniformfun$ and $\DensityDifferencefun$ is minimized when $\Uniformfun = \DensityDifferencefun$. Recall that the natural statistic are such that the exponential family is minimal. Therefore, $\Uniformfun = \DensityDifferencefun$ if and only if $\Theta = \ThetaStar$. Thus, $\ThetaStar \in \argmin_{\Theta \in \parameterSet}  \cL(\Theta)$,
	and it is a unique minimizer of $\cL(\Theta)$.
\end{proof}

\section{Proof of Theorem 4.2}
\label{appendix:proof of thm:consistency_normality}
In this Section, we prove Theorem \ref{thm:consistency_normality} by using the theory of $M$-estimation. In particular, observe that $\hThetan$ is an $M$-estimator i.e., $\hThetan$ is a sample average. Therefore, we invoke Theorem 4.1.1 and Theorem 4.1.3 of \cite{Amemiya1985} to prove the consistency and normality of $\hThetan$. We restate the Theorem below and then provide the proof. 

\theoremconsistencynormality*
\begin{proof}[Proof of Theorem \ref{thm:consistency_normality}] We divide the proof in two parts.\\
	
	{\bf Consistency. } We will first show that $\hThetan$ is asymptotically consistent. In order to show this, let us recall  Theorem 4.1.1 of \cite{Amemiya1985}.
	
	\cite[Theorem~4.1.1]{Amemiya1985}: Let $z_1, \cdots, z_n$ be i.i.d. samples of a random variable $\rvz$. Let $q(\rvz ; \theta)$ be some function of $\rvz$ parameterized by $\theta \in \Upsilon$. Let $\theta^*$ be the true underlying parameter. Define
	\begin{align}
	Q_n(\theta) = \frac{1}{n} \sum_{i = 1}^{n} q(z_i ; \theta) \qquad \text{and} \qquad
	\hthetan \in \argmin_{\theta \in \Upsilon} Q_n(\theta). \label{eq:m-est-con}
	\end{align} 
	Let the following be true.
	\begin{enumerate}[leftmargin=6mm, itemsep=-0.5mm]
		\item[(a)] $\Upsilon$ is compact,
		\item[(b)] $Q_n(\theta)$ converges uniformly in probability to a non-stochastic function $Q(\theta)$, 
		\item[(c)] $Q(\theta)$ is continuous, and
		\item[(d)] $Q(\theta)$ is uniquely minimized at $\theta^*$.
	\end{enumerate}
Then, $\hthetan$ is consistent for $\theta^*$ i.e., $\hthetan \stackrel{p}{\to} \theta^*$ as $n\to \infty$.\\

Letting $z \coloneqq \rvbx$, $\theta \coloneqq \Theta$, $\hthetan \coloneqq \hThetan$, $\theta^* \coloneqq \ThetaStar$, $\Upsilon = \parameterSet$, $q(z; \theta) \coloneqq \exp\big( -\big\langle \big\langle \Theta, \varPhi(\svbx) \big\rangle \big\rangle \big)$, and $Q_n(\theta) \coloneqq \cL_{n}(\Theta)$, it is sufficient to show the following:
\begin{enumerate}[leftmargin=6mm, itemsep=-0.5mm]
	\item[(a)] $\parameterSet$ is compact,
	\item[(b)] $\cL_{n}(\Theta)$ converges uniformly in probability to a non-stochastic function $\cL(\Theta)$, 
	\item[(c)] $\cL(\Theta)$ is continuous, and
	\item[(d)] $\cL(\Theta)$ is uniquely minimized at $\ThetaStar$.
\end{enumerate}

Let us show these one by one.
\begin{enumerate}[leftmargin=6mm, itemsep=-0.5mm]
	\item[(a)] We have $\parameterSet = \{\Theta : \bcR(\Theta) \leq \br\}$ which is bounded and closed. Therefore, $\parameterSet$ is compact.
	
	\item[(b)] Recall \cite[Theorem 2]{Jennrich1969}: Let $z_1, \cdots, z_n$ be i.i.d. samples of a random variable $\rvz$. Let $g(\rvz ; \theta)$ be a  function of $\theta$ parameterized by $\theta \in \Upsilon$.  Then, $n^{-1} \sum_t g(z_t , \theta)$ converges uniformly in probability to $\Expectation [ g(\rvz, \theta)]$ if 
	\begin{enumerate}[leftmargin=6mm]
		\item[(i)] $\Upsilon$ is compact, 
		\item[(ii)] $g(\rvz , \theta)$ is continuous at each $\theta \in \Upsilon$ with probability one, 
		\item[(iii)] $g(\rvz , \theta)$ is dominated by a function $G(\rvz)$ i.e., $| g(\rvz , \theta) | \leq G(\rvz)$, and
		\item[(iv)] $\Expectation[G(\rvz)] < \infty$.
	\end{enumerate}

	Using this theorem with $\rvz \coloneqq \rvbx$, $\theta \coloneqq \Theta$, $\Upsilon \coloneqq \parameterSet$, $g(\rvz, \theta) \coloneqq \exp\big( -\big\langle \big\langle \Theta, \varPhi(\svbx) \big\rangle \big\rangle \big)$, $G(\rvz) \coloneqq \exp(\br^T \bd)$ and \eqref{eq:bound_inner_product}, we conclude that $\cL_{n}(\Theta)$ converges to $\cL(\Theta)$ uniformly in probability.
	
	\item[(c)] $\exp\big( -\big\langle \big\langle \Theta, \varPhi(\svbx) \big\rangle \big\rangle \big)$ is a continuous function of $\Theta \in \parameterSet$. Further, $\DensityXTrue$ does not functionally depend on $\Theta$. Therefore, we have continuity of $\cL(\Theta)$ for all $\Theta \in \parameterSet$.
	
	\item[(d)] From Theorem \ref{theorem:GRISMe-KLD}, $\cL(\Theta)$ is uniquely minimized at $\ThetaStar$.
\end{enumerate}
		
Therefore, we have asymptotic consistency of $\hThetan$.\\

	{\bf Normality. }
	We will now show that $\hThetan$ is asymptotically normal. In order to show this, let us recall Theorem 4.1.3 of \cite{Amemiya1985}.
	
	\cite[Theorem~4.1.3]{Amemiya1985}: Let $z_1, \cdots, z_n$ be i.i.d. samples of a random variable $\rvz$. Let $q(\rvz ; \theta)$ be some function of $\rvz$ parameterized by $\theta \in \Upsilon$. Let $\theta^*$ be the true underlying parameter. Define
	\begin{align}
	Q_n(\theta) = \frac{1}{n} \sum_{i = 1}^{n} q(z_i ; \theta) \qquad \text{and} \qquad
	\hthetan \in \argmin_{\theta \in \Upsilon} Q_n(\theta). \label{eq:m-est-eff}
	\end{align} 
	Let the following be true.
	\begin{enumerate}[leftmargin=6mm, itemsep=-0.5mm]
		\item[(a)] $\hthetan$ is consistent for $\theta^*$, 
		\item[(b)] $\theta^*$ lies in the interior of the parameter space $\Upsilon$, 
		\item[(c)] $Q_n$ is twice continuously differentiable in an open and convex neighborhood of $\theta^*$, 
		\item[(d)] $\sqrt{n}\nabla Q_n(\theta)|_{\theta = \theta^*} \stackrel{d}{\to} {\cal N}({\bf 0}, A(\theta^*))$, and 
		\item[(e)] $\nabla^2 Q_n(\theta)|_{\theta = \hthetan} \stackrel{p}{\to} B(\theta^*)$ with $B(\theta)$ finite, non-singular, and continuous at $\theta^*$, 
	\end{enumerate}
Then, $\hthetan$ is normal for $\theta^*$ i.e., $\sqrt{n}( \hthetan - \theta^*)\stackrel{d}{\to} {\cal N}({\bf 0}, B^{-1}(\theta^*)A(\theta^*)B^{-1}(\theta^*))$.\\

Letting $z \coloneqq \rvbx$, $\theta \coloneqq \Theta$, $\hthetan \coloneqq \hThetan$, $\theta^* \coloneqq \ThetaStar$, $\Upsilon = \parameterSet$, $q(z; \theta) \coloneqq \exp\big( -\big\langle \big\langle \Theta, \varPhi(\svbx) \big\rangle \big\rangle \big)$, and $Q_n(\theta) \coloneqq \cL_{n}(\Theta)$, it is sufficient to show the following:
\begin{enumerate}[leftmargin=6mm, itemsep=-0.5mm]
	\item[(a)] $\hThetan$ is consistent for $\ThetaStar$, 
	\item[(b)] $\ThetaStar$ lies in the interior of the parameter space $\parameterSet$, 
	\item[(c)] $\cL_{n}$ is twice continuously differentiable in an open and convex neighborhood of $\ThetaStar$, 
	\item[(d)] $\sqrt{n}\nabla \cL_{n}(\vect(\Theta))|_{\Theta = \ThetaStar} \stackrel{d}{\to} {\cal N}({\bf 0}, A(\ThetaStar))$, and 
	\item[(e)] $\nabla^2 \cL_{n}(\vect(\Theta))|_{\Theta = \hThetan} \stackrel{p}{\to} B(\ThetaStar)$ with $B(\Theta)$ finite, non-singular, and continuous at $\ThetaStar$, 
\end{enumerate}

Let us show these one by one.

\begin{enumerate}[leftmargin=6mm, itemsep=-0.5mm]
			\item[(a)] We have established that $\hThetan$ is consistent for $\ThetaStar$ in the first half of the proof.
			\item[(b)] The assumption that $\ThetaStar \in \text{interior}(\parameterSet)$ is equivalent to $\ThetaStar$ belonging to the interior of $\parameterSet$.
			\item[(c)] Fix $u_1,u_2 \in [\dimOne]$, $v_1,v_2 \in [ \dimTwo]$, and $w_1,w_2 \in [ \dimThree]$. We have
			\begin{align}
			\frac{\partial^2 \cL_{n}(\Theta)}{\partial \Theta_{u_1v_1w_1} \partial\Theta_{u_2v_2w_2}}  = \frac{1}{n} \sum_{t = 1}^{n} \varPhi_{u_1v_1w_1}(\svbx^{(t)}) \varPhi_{u_2v_2w_2}(\svbx^{(t)}) \exp\big( -\big\langle \big\langle \Theta, \varPhi(\svbx^{(t)}) \big\rangle \big\rangle \big). 
			\end{align}
		Thus, $\partial^2 \cL_{n}(\Theta)/\partial \Theta_{u_1v_1w_1} \partial \Theta_{u_2v_2w_2}$ exists. Using the continuity of $\varPhi(\cdot)$ and 
		$\exp\big( -\big\langle \big\langle \Theta, \varPhi(\cdot) \big\rangle \big\rangle \big)$, we see that $\partial^2 \cL_{n}(\Theta)/\partial \Theta_{u_1v_1w_1} \partial \Theta_{u_2v_2w_2}$ is continuous in an open and convex neighborhood of $\ThetaStar$.
		\item[(d)] For any $u \in [\dimOne]$, $v \in [ \dimTwo]$ and $w \in [ \dimThree]$, define the random variable
		\begin{align}
		\rvx_{uvw}  = - \varPhi_{uvw}(\svbx) \exp\big( -\big\langle \big\langle \ThetaStar, \varPhi(\svbx) \big\rangle \big\rangle \big).
		\end{align}
		The component of the gradient of $\cL_{n}(\vect(\Theta))$ corresponding to $\Theta_{uvw}$
		evaluated at $\ThetaStar$ is given by
		\begin{align}
		\frac{\partial \cL_{n}(\ThetaStar)}{\partial \Theta_{uvw}}  = - \frac{1}{n} \sum_{t = 1}^{n} \varPhi_{uvw}(\svbx^{(t)}) \exp\big( -\big\langle \big\langle \ThetaStar, \varPhi(\svbx^{(t)}) \big\rangle \big\rangle \big). 
		\end{align}
				Each term in the above summation is distributed as the random variable $\rvx_{uvw} $.
				The random variable $\rvx_{uvw}$ has zero mean (see Lemma \ref{lemma:zero_expectation}). Using this and the multivariate central limit theorem \cite{Vaart2000}, we have
				\begin{align}
				\sqrt{n} \nabla \cL_{n}(\vect(\Theta)) |_{\Theta = \ThetaStar} \xrightarrow{d} {\cal N}({\bf 0}, A(\ThetaStar)),
				\end{align} 
				where $A(\ThetaStar)$ is the covariance matrix of $\vect\big(\varPhi(\rvbx)\exp\big( -\big\langle \big\langle \ThetaStar, \varPhi(\rvbx) \big\rangle \big\rangle \big)\big).$
	\item [(e)] We will start by showing that the following is true.
	\begin{align}
	\nabla^2 \cL_{n}(\vect(\Theta)) |_{\Theta = \hThetan} \xrightarrow{p} \nabla^2 \cL(\vect(\Theta)) |_{\Theta = \ThetaStar}.  \label{eq:ulln+cmt}
	\end{align}
	To begin with, using the uniform law of large numbers \cite[Theorem 2]{Jennrich1969} for any $\Theta \in \parameterSet$ results in
	\begin{align}
	\nabla^2 \cL_{n}(\vect(\Theta))  \xrightarrow{p} \nabla^2 \cL(\vect(\Theta)). \label{eq:ulln} 
	\end{align}
	Using the consistency of $\hThetan$ and the continuous mapping theorem, we have
	\begin{align}
	\nabla^2 \cL(\vect(\Theta)) |_{\Theta = \hThetan} \xrightarrow{p} \nabla^2 \cL(\vect(\Theta)) |_{\Theta = \ThetaStar}.  \label{eq:cmt}
	\end{align}
	Let $u_1,u_2 \in [\dimOne]$, $v_1,v_2 \in [ \dimTwo]$, and $w_1,w_2 \in [ \dimThree]$. From \eqref{eq:ulln} and \eqref{eq:cmt}, for any $\epsilon > 0$, for any $\delta > 0$, there exists integers $n_1 , n_2$ such that for $n \geq \max\{n_1,n_2\}$ we have,
	\begin{align}
	\Prob( | \partial^2 \cL_{n}(\hThetan)/\partial \Theta_{u_1v_1w_1} \partial \Theta_{u_2v_2w_2}  - \partial^2 \cL(\hThetan)/\partial \Theta_{u_1v_1w_1} \partial \Theta_{u_2v_2w_2}  | > \epsilon / 2 ) \leq \delta / 2
	\end{align}
	and 
	\begin{align}
	\Prob( | \partial^2 \cL(\hThetan)/\partial \Theta_{u_1v_1w_1} \partial \Theta_{u_2v_2w_2}  - \partial^2 \cL(\ThetaStar)/\partial \Theta_{u_1v_1w_1} \partial \Theta_{u_2v_2w_2}  | > \epsilon / 2 ) \leq \delta / 2.
	\end{align} 
	Now for $n \geq \max\{n_1,n_2\}$, using the triangle inequality we have
	\begin{align}
	\Prob( | \partial^2 \cL_{n}(\hThetan)/\partial \Theta_{u_1v_1w_1} \partial \Theta_{u_2v_2w_2}  - \partial^2 \cL(\ThetaStar)/\partial \Theta_{u_1v_1w_1} \partial \Theta_{u_2v_2w_2}  | > \epsilon)  \leq \delta / 2 + \delta / 2 = \delta.
	\end{align}
	Thus, we have \eqref{eq:ulln+cmt}. Using the definition of $\cL(\Theta)$, we have
	\begin{align}
	\partial^2 \cL(\ThetaStar)/\partial \Theta_{u_1v_1w_1} \partial \Theta_{u_2v_2w_2} & = \Expectation \Big[ \varPhi_{u_1v_1w_1}(\rvbx) \varPhi_{u_2v_2w_2}(\rvbx)\exp\big( -\big\langle \big\langle \ThetaStar, \varPhi(\rvbx) \big\rangle \big\rangle \big)  \Big] \\
	& \stackrel{(b)}{=} \Expectation \Big[ \varPhi_{u_1v_1w_1}(\rvbx) \varPhi_{u_2v_2w_2}(\rvbx)\exp\big( -\big\langle \big\langle \ThetaStar, \varPhi(\rvbx) \big\rangle \big\rangle \big) \Big]  \\ 
	& \qquad - \Expectation \Big[ \varPhi_{u_1v_1w_1}(\rvbx)\Big]  \Expectation \Big[ \varPhi_{u_2v_2w_2}(\rvbx)\exp\big( -\big\langle \big\langle \ThetaStar, \varPhi(\rvbx) \big\rangle \big\rangle\big)\Big]\\
	& = \text{cov} \Big(  \varPhi_{u_1v_1w_1}(\rvbx) , \varPhi_{u_2v_2w_2}(\rvbx)\exp\big( -\big\langle \big\langle \ThetaStar, \varPhi(\rvbx) \big\rangle \big\rangle \big)\Big),
	\end{align}	
	where (b) follows because $\Expectation \big[ \varPhi_{u_2v_2w_2}(\rvbx)\exp\big( -\big\langle \big\langle \ThetaStar, \varPhi(\rvbx) \big\rangle \big\rangle\big)\big] = 0$ for any $u_2 \in [\dimOne]$, $v_2 \in [ \dimTwo]$, and $w_2 \in [ \dimThree]$ from Lemma \ref{lemma:zero_expectation}.	Therefore, we have
	\begin{align}
	\nabla^2 \cL_{n}(\vect(\Theta)) |_{\Theta = \hThetan} \xrightarrow{p} B(\ThetaStar),
	\end{align}
	where $B(\ThetaStar)$ is the cross-covariance matrix of $\vect(\varPhi(\rvbx))$ and $\vect\big(\varPhi(\rvbx) \exp\big( -\big\langle \big\langle \ThetaStar, \varPhi(\rvbx) \big\rangle \big\rangle \big)\big)$. Finiteness and continuity of $\varPhi(\rvbx)$ and $\varPhi(\rvbx)\exp\big( -\big\langle \big\langle \ThetaStar, \varPhi(\rvbx) \big\rangle \big\rangle \big)$ implies the finiteness and continuity of $B(\ThetaStar)$. By assumption,  the cross-covariance matrix of $\vect(\varPhi(\rvbx))$ and $\vect\big(\varPhi(\rvbx) \exp\big( -\big\langle \big\langle \ThetaStar, \varPhi(\rvbx) \big\rangle \big\rangle \big)\big)$ is invertible.
\end{enumerate}	
Therefore, we have the asymptotic normality of $\hThetan$.
\end{proof}

\section{Restricted strong convexity of the loss function} 
\label{appendix:rsc}
In this Section, we will show that, with enough samples, the loss function obeys the restricted strong convexity property with high probability. This result will in turn allow us to prove Theorem \ref{thm:finite_sample} in Appendix \ref{appendix_proof_finite_sample}

We will first state the main result of this Section (Proposition \ref{prop:rsc_GISMe}). Next, we will introduce the notion of correlation for the centered natural statistics and provide a supporting Lemma wherein we will bound the deviation between the true correlation and the empirical correlation. Finally, we will prove Proposition \ref{prop:rsc_GISMe}.

Consider any $\Theta \in \parameterSet$. 
Let $\Delta = \Theta - \ThetaStar$. 
Define the residual of the first-order Taylor expansion as 
\begin{align}
\delta \cL_{n}(\Delta, \ThetaStar) = \cL_{n}(\ThetaStar + \Delta) - \cL_{n}(\ThetaStar)  - \langle \langle \nabla \cL_{n}(\ThetaStar),\Delta \rangle\rangle.  \label{eq:residual}
\end{align}
\begin{restatable}{proposition}{proprsc}\label{prop:rsc_GISMe}
	Let Assumptions \ref{bounds_parameter}, \ref{bounds_statistics}, \ref{bounds_statistics_maximum} and \ref{lambdamin} be satisfied.
	For any $\deltaThree \in (0,1)$, the residual defined in \eqref{eq:residual} satisfies
	\begin{align}
	\delta \cL_{n}(\Delta, \ThetaStar) \geq \frac{\lambdaMin \exp(-\br^T\bd )}{4(1 + \br^T\bd)} \|\Delta\|^2_{\mathrm{T}}, 
	\end{align}
	with probability at least $1-\deltaThree$ as long as 
	\begin{align}
	n > \frac{8 \phiMax^4 \dimOne^2  \dimTwo^2 \dimThree^3 }{\lambdaMin^2}\log\Big(\frac{2 \dimOne^2  \dimTwo^2 \dimThree^3 }{\deltaThree}\Big). 
	\end{align}
\end{restatable}

\subsection{Correlation between centered natural statistics}
\label{subsec:correlation between centered natural statistics}
For any $u_1,u_2 \in [\dimOne]$, $v_1,v_2 \in [ \dimTwo]$, and $w_1,w_2 \in [ \dimThree]$, let $H_{u_1v_1w_1u_2v_2w_2}$ denote the correlation between $\varPhi_{u_1v_1w_1}(\rvbx)$ and $\varPhi_{u_2v_2w_2}(\rvbx)$ defined as
\begin{align}
H_{u_1v_1w_1u_2v_2w_2} = \Expectation \big[\varPhi_{u_1v_1w_1}(\rvbx)\varPhi_{u_2v_2w_2}(\rvbx)\big],  \label{eq:population_correlation}
\end{align}
and let $\bH = [H_{u_1v_1w_1u_2v_2w_2}] \in \Reals^{[\dimOne] \times [ \dimTwo] \times [ \dimThree]  \times [\dimOne] \times [ \dimTwo] \times [ \dimThree] }$ be the corresponding correlation tensor. 
Similarly, we define $\hbH$  based on the empirical estimates of the correlation
\begin{align}
\hH_{u_1v_1w_1u_2v_2w_2} = \frac{1}{n} \sum_{t=1}^{n} \varPhi_{u_1v_1w_1}(\svbx^{(t)})\varPhi_{u_2v_2w_2}(\svbx^{(t)}). \label{eq:empirical_correlation}
\end{align}

The following lemma bounds the deviation between the true correlation and the empirical correlation.
\begin{lemma} \label{lemma:correlation_concentration}
	Consider any $u_1,u_2 \in [\dimOne]$, $v_1,v_2 \in [ \dimTwo]$, and $w_1,w_2 \in [ \dimThree]$. Let Assumption \ref{bounds_statistics_maximum} be satisfied. Then, we have for any $\epsTwo > 0$,
	\begin{align}
	|\hH_{u_1v_1w_1u_2v_2w_2} - H_{u_1v_1w_1u_2v_2w_2}| < \epsTwo,
	\end{align}
	with probability at least $1 - \deltaTwo$ as long as
	\begin{align}
	n > \frac{2 \phiMax^4}{\epsTwo^2}\log\Big(\frac{2 \dimOne^2  \dimTwo^2 \dimThree^2}{\deltaTwo}\Big).
	\end{align}
\end{lemma}
\begin{proof}[Proof of Lemma~\ref{lemma:correlation_concentration}]
	Fix $u_1,u_2 \in [\dimOne]$, $v_1,v_2 \in [ \dimTwo]$, and $w_1,w_2 \in [ \dimThree]$. 
	The random variable defined as $Y_{u_1v_1w_1u_2v_2w_2} \coloneqq \varPhi_{u_1v_1w_1}(\rvbx)\varPhi_{u_2v_2w_2}(\rvbx)$ satisfies $|Y_{u_1v_1w_1u_2v_2w_2}| \leq \phiMax^2$ (from Assumption \ref{bounds_statistics_maximum}). 
	Using the Hoeffding inequality we get
	\begin{align}
	\Prob \left( |\hH_{u_1v_1w_1u_2v_2w_2} - H_{u_1v_1w_1u_2v_2w_2}| > \epsTwo \right) < 2\exp \left(-\frac{n \epsTwo^2}{2\phiMax^4}\right).
	\end{align}
	The proof follows by using the union bound over all $u_1,u_2 \in [\dimOne]$, $v_1,v_2 \in [ \dimTwo]$, and $w_1,w_2 \in [ \dimThree]$.
\end{proof}

\subsection{Proof of Proposition~\ref{prop:rsc_GISMe}}

\begin{proof}[Proof of Proposition~\ref{prop:rsc_GISMe}]
	
	First, we will simplify the gradient of $\cL_{n}(\Theta)$\footnote[6]{Ideally, one would consider the gradient of $\cL_{n}(\vect({\Theta}))$. However, for the ease of the exposition we abuse the terminology.} evaluated at $\ThetaStar$. 
	For any $u \in [\dimOne]$, $v \in [ \dimTwo]$ and $w \in [ \dimThree]$, the component of the gradient of $\cL_{n}(\Theta)$ corresponding to $\Theta_{uvw}$
	evaluated at $\ThetaStar$ is given by
	\begin{align}
	\frac{\partial \cL_{n}(\ThetaStar)}{\partial \Theta_{uvw}}  = -\frac{1}{n} \sum_{t = 1}^{n} \varPhi_{uvw}(\svbx^{(t)}) \exp\big( -\big\langle\big\langle \ThetaStar, \varPhi(\svbx^{(t)}) \big\rangle \big\rangle \big).  \label{eq:gradient-GISMe}
	\end{align}
	
	We will now provide the desired lower bound on the residual. Substituting \eqref{eq:sampleGISMe} and \eqref{eq:gradient-GISMe} in \eqref{eq:residual}, we have
	\begin{align}
	\delta \cL_{n}(\Delta, \ThetaStar) &= \frac{1}{n} \sum_{t = 1}^{n}  \exp\big( \hspace{-0.5mm}-\hspace{-0.5mm}\big\langle\big\langle \ThetaStar, \varPhi(\svbx^{(t)}) \big\rangle \big\rangle \big) \times \Big[ \exp\big( \hspace{-0.5mm}-\hspace{-0.5mm}\big\langle\big\langle \Delta, \varPhi(\svbx^{(t)}) \big\rangle \big\rangle \big) - \hspace{-0.5mm} 1 \hspace{-0.5mm} + \big\langle \big\langle \Delta, \varPhi(\svbx^{(t)}) \big\rangle \big\rangle \Big]  \\
	& \stackrel{(a)}{\geq} \exp(-\br^T\bd) \times \frac{1}{n} \sum_{t = 1}^{n}  \Big[ \exp\big( -\big\langle\big\langle \Delta, \varPhi(\svbx^{(t)}) \big\rangle \big\rangle \big) - 1 +  \big\langle\big\langle \Delta, \varPhi(\svbx^{(t)}) \big\rangle \big\rangle \Big]  \\	
	& \stackrel{(b)}{\geq} \exp(-\br^T\bd) \times \frac{1}{n} \sum_{t = 1}^{n}  \frac{\big|\big\langle\big\langle \Delta, \varPhi(\svbx^{(t)}) \big\rangle\big\rangle \big|^2}{2+\big|\big\langle\big\langle \Delta, \varPhi(\svbx^{(t)}) \big\rangle\big\rangle \big|}\\
	& \stackrel{(c)}{\geq} \frac{\exp(-\br^T\bd)}{2+2\br^T\bd} \times \frac{1}{n} \sum_{t = 1}^{n}  \big|\big\langle\big\langle \Delta, \varPhi(\svbx^{(t)}) \big\rangle\big\rangle\big|^2\\
	& \stackrel{(d)}{=} \frac{\exp(-\br^T\bd)}{2+2\br^T\bd} \times \sum_{u_1=1}^{\dimOne}\sum_{v_1=1}^{ \dimTwo}\sum_{w_1=1}^{\dimThree}  \sum_{u_2=1}^{\dimOne}\sum_{v_2=1}^{ \dimTwo}\sum_{w_2=1}^{\dimThree} 
	\Delta_{u_1v_1w_1} \hH_{u_1v_1w_1u_2v_2w_2} \Delta_{u_2v_2w_2}\\
	& = \frac{\exp(-\br^T\bd)}{2+2\br^T\bd} \times \sum_{u_1=1}^{\dimOne}\sum_{v_1=1}^{ \dimTwo}\sum_{w_1=1}^{\dimThree}  \sum_{u_2=1}^{\dimOne}\sum_{v_2=1}^{ \dimTwo}\sum_{w_2=1}^{\dimThree} 
	 \Delta_{u_1v_1w_1} \times \\ 
	 & \qquad \qquad \qquad \qquad [H_{u_1v_1w_1u_2v_2w_2}  + \hH_{u_1v_1w_1u_2v_2w_2} - H_{u_1v_1w_1u_2v_2w_2}] \Delta_{u_2v_2w_2},
	\end{align}
	where $(a)$ follows because $ - \big\langle\big\langle \Theta, \varPhi(\svbx) \big\rangle \big\rangle \geq - \br^T\bd$ from \eqref{eq:bound_inner_product}, $(b)$ follows because $e^{-z} - 1 + z \geq \frac{z^2}{2 + |z|}$ for any $z \in \Reals$, $(c)$ follows from \eqref{eq:bound_inner_product}, and $(d)$ follows from \eqref{eq:empirical_correlation}.\\
	
	Let the number of samples satisfy
	\begin{align}
	n > \frac{8 \phiMax^4 \dimOne^2  \dimTwo^2 \dimThree^2 }{\lambdaMin^2}\log\Big(\frac{2 \dimOne^2  \dimTwo^2 \dimThree^2 }{\deltaThree}\Big). 
	\end{align}
	Using Lemma \ref{lemma:correlation_concentration} with $\epsTwo= \frac{\lambdaMin}{2\dimOne \dimTwo \dimThree}$ and $\deltaTwo = \deltaThree$, and the triangle inequality, we have the following with probability at least $1-\deltaThree$
	\begin{align}
	\delta \cL_{n}(\Delta, \ThetaStar) &\geq \frac{\exp(-\br^T\bd)}{2+2\br^T\bd} \times \bigg[ \sum_{u_1=1}^{\dimOne}\sum_{v_1=1}^{ \dimTwo}\sum_{w_1=1}^{\dimThree}  \sum_{u_2=1}^{\dimOne}\sum_{v_2=1}^{ \dimTwo}\sum_{w_2=1}^{\dimThree} 
	\Delta_{u_1v_1w_1} H_{u_1v_1w_1u_2v_2w_2} \Delta_{u_2v_2w_2} 
	\\ & \qquad \qquad \qquad \qquad \qquad \qquad
	- \frac{\lambdaMin}{2\dimOne \dimTwo \dimThree} \|\Delta\|^2_{1,1,1}  \bigg]\\
	&\stackrel{(a)}{\geq} \frac{\exp(-\br^T\bd)}{2+2\br^T\bd} \times \bigg[ \sum_{u_1=1}^{\dimOne}\sum_{v_1=1}^{ \dimTwo}\sum_{w_1=1}^{\dimThree}  \sum_{u_2=1}^{\dimOne}\sum_{v_2=1}^{ \dimTwo}\sum_{w_2=1}^{\dimThree} 
	  \Delta_{u_1v_1w_1} H_{u_1v_1w_1u_2v_2w_2} \Delta_{u_2v_2w_2} 
	  \\ & \qquad \qquad \qquad \qquad \qquad \qquad
	  - \frac{\lambdaMin}{2} \|\Delta\|^2_{\mathrm{T}}  \bigg]\\
	& \stackrel{(b)}{=} \frac{\exp(-\br^T\bd)}{2+2\br^T\bd} \times \Big[ \vect(\Delta) \Expectation[\vect(\varPhi(\rvbx))\vect(\varPhi(\rvbx))^T]\vect(\Delta)^T  - \frac{\lambdaMin}{2} \|\Delta\|^2_{\mathrm{T}}  \Big]\\
	& \stackrel{(c)}{\geq} \frac{\exp(-\br^T\bd)}{2+2\br^T\bd} \times \Big[ \lambdaMin \|\vect(\Delta)\|^2_2   - \frac{\lambdaMin}{2} \|\Delta\|^2_{\mathrm{T}}  \Big]\\
	& \stackrel{(d)}{=} \frac{\exp(-\br^T\bd)}{2+2\br^T\bd} \times \frac{\lambdaMin}{2} \|\Delta\|^2_{\mathrm{T}},
	\end{align}
	where $(a)$ follows because $\|\Delta\|_{1,1,1} \leq \sqrt{\dimOne \dimTwo \dimThree} \|\Delta\|_\mathrm{{T}}$, $(b)$ follows from \eqref{eq:population_correlation}, $(c)$ follows from the Courant-Fischer theorem (because $\Expectation[\vect(\varPhi(\rvbx))\vect(\varPhi(\rvbx))^T]$ is a symmetric matrix) and Assumption \ref{lambdamin}, and $(d)$ follows because $\|\vect(\Delta)\|_2 = \|\Delta\|_{\mathrm{T}}$.
\end{proof}

\section{Bounds on the tensor maximum norm of the gradient of the loss function} 
\label{appendix:bounds on the gradient of the GISMe}
In this Section, we will show that, with enough samples, the tensor maximum norm of the gradient of the loss function evaluated at the true natural parameter is bounded with high probability. This result will allow us to prove Theorem \ref{thm:finite_sample} in Appendix \ref{appendix_proof_finite_sample}.

We will first state the main result of this Section (Proposition \ref{prop:gradient-concentration-GISMe}). Next, we will provide a supporting Lemma wherein we show that the expected value of a random variable of interest is zero. Finally, we will prove Proposition \ref{prop:gradient-concentration-GISMe}.
\begin{proposition} \label{prop:gradient-concentration-GISMe}
	Let Assumptions \ref{bounds_parameter}, \ref{bounds_statistics} and \ref{bounds_statistics_maximum} be satisfied. For any $\deltaFour \in (0,1)$, any $\epsFour > 0$, the components of the gradient of the loss function $\cL_{n}(\Theta)$\footnote[7]{Ideally, one would consider the gradient of $\cL_{n}(\vect({\Theta}))$. However, for the ease of the exposition we abuse the terminology.} evaluated at $\ThetaStar$ are bounded from above as 
	\begin{align}
	\| \nabla \cL_{n}(\ThetaStar) \|_{\max} \leq \epsFour,
	\end{align}
	with probability at least $1 - \deltaFour$ as long as
	\begin{align}
	n > \frac{2 \phiMax^2\exp(2\br^T\bd)}{\epsFour^2}\log\Big(\frac{2\dimOne \dimTwo \dimThree}{\deltaFour}\Big).
	\end{align}
\end{proposition}

\subsection{Supporting Lemma for Proposition \ref{prop:gradient-concentration-GISMe}}
\begin{lemma}\label{lemma:zero_expectation}
	For any $u \in [\dimOne]$, $v \in [ \dimTwo]$ and $w \in [ \dimThree]$, define the random variable
	\begin{align}
	\rvx_{uvw}  = - \varPhi_{uvw}(\svbx) \exp\big( -\big\langle \big\langle \ThetaStar, \varPhi(\svbx) \big\rangle \big\rangle \big). \label{xij_definition}
	\end{align}
	We have 
	\begin{align}
	\Expectation[\rvx_{uvw}] = 0,
	\end{align}
	where the expectation is with respect to $\DensityXTrue$.
\end{lemma}
\begin{proof}[Proof of Lemma \ref{lemma:zero_expectation}]
	Fix any $u \in [\dimOne]$, $v \in [ \dimTwo]$ and $w \in [ \dimThree]$. Using \eqref{xij_definition}, we have
	\begin{align}
	\Expectation[\rvx_{uvw}] & = -\int_{\svbx \in \cX} \hspace{-2mm}\DensityXTrue \varPhi_{uvw}(\svbx) \exp\big( -\big\langle \big\langle \ThetaStar, \varPhi(\svbx) \big\rangle \big\rangle \big) d\svbx \stackrel{(a)}{=} \frac{-\int_{\svbx \in \cX}  \varPhi_{uvw}(\svbx)  d\svbx}{\int_{\svby \in \cX} \exp\big( \big\langle \big\langle \ThetaStar, \varPhi(\svby) \big\rangle \big\rangle \big) d\svby} \\
	& \stackrel{(b)}{=} 0,
	\end{align}
	where $(a)$ follows from the definition of $\DensityXTrue$, and because $\Expectation_{\Uniform} [\Phi(\rvbx)]$ is a constant, and $(b)$ follows because $\int_{\svbx \in \cX} \varPhi(\svbx)  d\svbx = 0$ from Definition \ref{def:css}
\end{proof}

\subsection{Proof of Proposition \ref{prop:gradient-concentration-GISMe}}
\begin{proof}[Proof of Proposition \ref{prop:gradient-concentration-GISMe}]
	Fix $u \in [\dimOne]$, $v \in [ \dimTwo]$ and $w \in [ \dimThree]$. We will start by simplifying the gradient of the $\cL_n(\Theta)$ evaluated at $\ThetaStar$. The component of the gradient of $\cL_{n}(\Theta)$ corresponding to $\Theta_{uvw}$
	evaluated at $\ThetaStar$ is given by
	\begin{align}
	\frac{\partial \cL_{n}(\ThetaStar)}{\partial \Theta_{uvw}}  = - \frac{1}{n} \sum_{t = 1}^{n} \varPhi_{uvw}(\svbx^{(t)}) \exp\Big( -\Big\langle \Big\langle \ThetaStar, \varPhi(\svbx^{(t)}) \Big\rangle \Big\rangle \Big).  \label{eq:gradient_GISMe_ThetaStar}
	\end{align}
	Each term in the above summation is distributed as the random variable $\rvx_{uvw}$ (see \eqref{xij_definition}). The random variable $\rvx_{uvw}$ has zero
	mean (see Lemma \ref{lemma:zero_expectation}) and satisfies $
	|\rvx_{uvw}| \leq \phiMax  \exp(\br^T\bd)$ (from Assumption \ref{bounds_statistics_maximum} and \eqref{eq:bound_inner_product}). Using the Hoeffding's inequality, we have
	\begin{align}
	\Prob \left( \bigg|\frac{\partial \cL_{n}(\ThetaStar)}{\partial \Theta_{uvw}}\bigg| > \epsFour \right) < 2\exp \left(-\frac{n \epsFour^2}{2\phiMax^2 \exp(2\br^T\bd)}\right). \label{eq:Hoeffding}
	\end{align}
	The proof follows by using \eqref{eq:Hoeffding} and the union bound over all $u \in [\dimOne]$, $v \in [ \dimTwo]$ and $w \in [ \dimThree]$.
\end{proof}

\section{Proof of Theorem 4.3}
\label{appendix_proof_finite_sample}
	In this Section, we will prove Theorem \ref{thm:finite_sample}. We restate the Theorem below and then provide the proof.
	\theoremfinite*
	\begin{proof}[Proof of Theorem \ref{thm:finite_sample}]
		Let the number of samples satisfy
		\begin{align}
		n & \geq \max  \bigg\{\frac{8 \phiMax^4 \dimOne^2  \dimTwo^2 \dimThree^2 }{\lambdaMin^2}\log\Big(\frac{4 \dimOne^2  \dimTwo^2 \dimThree^2 }{\delta}\Big), \\
		& \qquad \frac{2^9 \phiMax^2  \dimOne^2  \dimTwo^2 (\br^T\bg)^2(1 + \br^T\bd)^2  \exp(4\br^T\bd )}{\alpha^4\lambdaMin^2}\log\Big(\frac{4\dimOne \dimTwo \dimThree}{\delta}\Big)\bigg\} \\
		& \stackrel{(a)}{\approx} O\Big(\frac{ \dimOne^2  \dimTwo^2 }{\alpha^4 \lambdaMin^2}\log\big(\frac{\dimOne \dimTwo}{\delta}\big)\Big) \approx O\Big(\mathrm{poly}\Big( \frac{\dimOne \dimTwo}{\alpha}\Big)\Big).
		\end{align}
		where $(a)$ follows because $\dimThree, \phiMax, \br, \bg, \bd = O(1)$.
		
		Let $\Delta = \hThetaEps - \ThetaStar$. 
		Define the residual of the first-order Taylor expansion as 
		\begin{align}
		\delta \cL_{n}(\Delta, \ThetaStar) = \cL_{n}(\ThetaStar + \Delta) - \cL_{n}(\ThetaStar)  - \langle\langle\nabla \cL_{n}(\ThetaStar),\Delta \rangle \rangle.  \label{eq:residual1}
		\end{align}
		Let $\nabla \cL^{(i)}_{n}(\ThetaStar)$ denote the $i^{th}$ slice of $\nabla \cL_{n}(\ThetaStar)$. From the definition of an $\epsilon$-optimal solution of $\hThetan$, we have 
		\begin{align}
		\epsilon &\geq \cL_{n}(\hThetaEps) - \min_{\Theta \in \parameterSet} \cL_{n}(\Theta)  \\
		&\geq \cL_{n}(\hThetaEps) - \cL_{n}(\ThetaStar) \\
		&\stackrel{(a)}{=} \langle\langle \nabla \cL_{n}(\ThetaStar),\hThetaEps - \ThetaStar \rangle\rangle +  \delta \cL_{n}(\Delta, \ThetaStar)\\
		&\stackrel{(b)}{=} \sum_{i = 1}^{\dimThree} \langle\nabla \cL^{(i)}_{n}(\ThetaStar),\hThetaEps^{(i)} - \ThetaStari \rangle +  \delta \cL_{n}(\Delta, \ThetaStar)\\
		&\stackrel{(c)}{\geq} - \sum_{i = 1}^{\dimThree}  \cR^*_i(\nabla \cL^{(i)}_{n}(\ThetaStar)) \times \cR(\hThetaEps^{(i)} - \ThetaStari ) +  \delta \cL_{n}(\Delta, \ThetaStar)\\
		&\stackrel{(d)}{\geq} - 2\sum_{i = 1}^{\dimThree}  \cR^*_i(\nabla \cL^{(i)}_{n}(\ThetaStar)) \times r_i +  \delta \cL_{n}(\Delta, \ThetaStar)\\
		&\stackrel{(e)}{\geq} - 2 \dimOne  \dimTwo\sum_{i = 1}^{\dimThree}  g_i \times  \times \|\nabla \cL^{(i)}_{n}(\ThetaStar)\|_{\max} \times r_i +  \delta \cL_{n}(\Delta, \ThetaStar)\\
		&\stackrel{(f)}{\geq} - 2 \dimOne  \dimTwo \|\nabla \cL_{n}(\ThetaStar)\|_{\max} \sum_{i = 1}^{\dimThree}  g_i \times  r_i +  \delta \cL_{n}(\Delta, \ThetaStar),
		\end{align}
		where $(a)$ follows from \eqref{eq:residual1}, $(b)$ follows from the definitions of a slice of a tensor, tensor inner product, and Frobenius inner product, $(c)$ follows from the definition of a dual norm, $(d)$ follows because $\cR(\hThetaEps^{(i)} - \ThetaStari ) \leq \cR(\hThetaEps^{(i)}) + \cR(\ThetaStari ) \leq 2 r_i$ from Assumption \ref{bounds_parameter}, $(e)$ follows from Property \ref{property:norms} in Section \ref{sec:main results}, and $(f)$ follows because $\|\nabla \cL^{(i)}_{n}(\ThetaStar)\|_{\max} \leq \|\nabla \cL_{n}(\ThetaStar)\|_{\max} ~\forall i \in [\dimThree]$.
		
		Using Proposition \ref{prop:rsc_GISMe} with $\deltaThree = \frac{\delta}{2}$, and Proposition \ref{prop:gradient-concentration-GISMe} with $\deltaFour = \frac{\delta}{2}$, we have the following with probability at least $1-\delta$. 
		\begin{align}
		\epsilon & \geq - 2  \dimOne   \dimTwo  \epsFour \times \br^T \bg + \frac{\lambdaMin \exp(-\br^T\bd )}{4(1 + \br^T\bd)} \|\Delta\|^2_{\mathrm{T}}. 
		\end{align}
		This can be rearranged 
		\begin{align}
		\|\Delta\|^2_{\mathrm{T}}  & \leq \frac{\epsilon + 2 \dimOne   \dimTwo  \epsFour \times \br^T \bg}{\lambdaMin} \times 4(1 + \br^T\bd) \exp(\br^T\bd ). \label{eq:rearranged_delta}
		\end{align}
		Now, let
		\begin{align}
		\epsilon = \frac{\alpha^2 \lambdaMin}{8(1 + \br^T\bd)  \exp(\br^T\bd )} \qquad  \text{and} \qquad \epsFour = \frac{\alpha^2 \lambdaMin}{16 \dimOne   \dimTwo   \times \br^T \bg \times (1 + \br^T\bd) \times \exp(\br^T\bd )}. \label{eq:epsilons}
		\end{align}
		Plugging in $\epsilon$ and $\epsFour$ from \eqref{eq:epsilons} in \eqref{eq:rearranged_delta}, we obtain that
		\begin{align}
		\|\Delta\|_{\mathrm{T}} \leq \alpha.
		\end{align}
	The computational cost of the operation $\Theta_{(t)} - \eta  \nabla \cL_{n}(\Theta_{(t)}) - \Theta$ in Algorithm \ref{alg:GradientDescent}
	is of the order $\dimOne \dimTwo  n$ (because $\dimThree = O(1)$). 
	Therefore, the computational cost of the step $\Theta_{(t+1)} \leftarrow \argmin_{\Theta \in \parameterSet} \| \Theta_{(t)} - \eta  \nabla \cL_{n}(\Theta_{(t)}) - \Theta\|$ of Algorithm \ref{alg:GradientDescent} is of the order $\max\{\dimOne \dimTwo n, c(\parameterSet)\}$. 
	 From Lemma \ref{lemma:gradient_descent}, with $\epsilon = O(\alpha^2 \lambdaMin)$, Algorithm \ref{alg:GradientDescent} returns an $\epsilon$-optimal solution $\hThetaEps$ as long as $\tau = O\Big(\mathrm{poly}\Big( \frac{\dimOne \dimTwo}{\alpha^2 \lambdaMin}\Big)\Big).$
	 Therefore, the total computational cost scales as $	 O\Big(\frac{\dimOne \dimTwo}{\alpha^2 \lambdaMin}\max\big(\dimOne \dimTwo n, c(\parameterSet)\big)\Big)$.  Whenever the cost of projection onto $\parameterSet$ is $O\big(\mathrm{poly}( \dimOne \dimTwo)\big)$, we have the total computational cost scaling as $ O\Big(\mathrm{poly}\Big( \frac{\dimOne \dimTwo}{\alpha}\Big)\Big)$.
	\end{proof}

\section{Computational cost for the example constraints on the natural parameters}
\label{appendix:computational_cost}
In this Section, we provide Corollary \ref{corr_sparse_comp}, Corollary \ref{corr_low_rank_comp}, and Corollary \ref{corr_sparse_low_rank_comp}. These Corollaries provide the computational cost to produce an $\epsilon$-optimal solution of $\hThetan$ for sparse decomposition of $\Theta$, low-rank decomposition of $\Theta$, and sparse-plus-low-rank decomposition of $\Theta$. respectively. Recall the convex relaxations of these constraints from Section \ref{subsec:examples}.

\subsection{Sparse Decomposition}
\label{appendix:corr_sparse_comp}
\begin{restatable}{corollary}{corrSparse}\label{corr_sparse_comp}(Sparse decomposition)
	Suppose $\ThetaStar$ has a sparse decomposition i.e., $\ThetaStar = (\ThetaStarOne)$ and $\|\ThetaStarOne\|_{1,1} \leq r_1$. Let Assumptions \ref{bounds_parameter}, \ref{bounds_statistics}, \ref{bounds_statistics_maximum}, and \ref{lambdamin} be satisfied. Let 
	\begin{align}
	n & \geq O\bigg(\frac{ \dimOne^2  \dimTwo^2 }{\alpha^4 \lambdaMin^2}\log\Big(\frac{\dimOne \dimTwo}{\delta}\Big)\bigg).
	\end{align}
	Let $\eta = 1/\dimOne \dimTwo \dimThree \phiMax^2\exp(r_1d_1)$ and $\Theta^{(0)} = \boldsymbol{0}$. Then, Algorithm \ref{alg:GradientDescent} is guaranteed to produce an $\epsilon$-optimal solution $\hThetaEps$ such that $\|\hThetaEps - \ThetaStar\|_{\mathrm{T}} \leq \alpha$,
	with probability at least $1-\delta$ and with number of computations of the order
	\begin{align}
	O\bigg(\frac{ \dimOne^4  \dimTwo^4 }{\alpha^6 \lambdaMin^3}\log\Big(\frac{\dimOne \dimTwo }{\delta}\Big)\bigg).
	\end{align}
\end{restatable}
\begin{proof}[Proof of Corollary \ref{corr_sparse_comp} ]
	The computational cost of projecting on the $L_{1,1}$ ball is $O(\dimOne \dimTwo)$ (see \cite{DuchiSSC2008} and note $\dimThree = O(1)$). The computational cost of the operation $\Theta_{(t)} - \eta  \nabla \cL_{n}(\Theta_{(t)}) - \Theta$ is $O(\dimOne \dimTwo n)$ (because $\dimThree = O(1)$).  Therefore, the computational cost of the step $\Theta_{(t+1)} \leftarrow \argmin_{\Theta \in \parameterSet} \| \Theta_{(t)} - \eta  \nabla \cL_{n}(\Theta_{(t)}) - \Theta\|$ of Algorithm \ref{alg:GradientDescent} is $O(\dimOne \dimTwo n)$. 
	
	From Lemma \ref{lemma:gradient_descent}, Algorithm \ref{alg:GradientDescent} returns an $\epsilon$-optimal solution $\hThetaEps$ as long as
	\begin{align}
	\tau \geq \frac{2\dimOne \dimTwo\phiMax^2\exp(\br^T\bd)}{\epsilon} \| \hTheta_n\|^2_{\mathrm{T}}.
	\end{align}
	Also, $\| \hTheta_n\|^2_{\mathrm{T}} = \| \hTheta_n^{(1)}\|^2_{\mathrm{F}} \leq \| \hTheta_n^{(1)}\|^2_{1,1} \leq r_1^2.$ Combining everything, the computational cost scales as $O\Big(\frac{ \dimOne^2  \dimTwo^2  n}{\epsilon} \Big)$.
	Using Theorem \ref{thm:finite_sample}, and plugging in $n = O\Big(\frac{ \dimOne^2  \dimTwo^2 }{\alpha^4 \lambdaMin^2}\log\Big(\frac{\dimOne \dimTwo }{\delta}\Big)\Big)$ and $\epsilon = O(\alpha^2 \lambdaMin)$ completes the proof. 
\end{proof}

\subsection{Low-rank decomposition}
\label{appendix:corr_low_rank_comp}
\begin{restatable}{corollary}{corrLowRank}\label{corr_low_rank_comp}(Low-rank decomposition)
	Suppose $\ThetaStar$ has a low-rank decomposition i.e., $\ThetaStar = (\ThetaStarOne)$ and $\|\ThetaStar\|_{\star} \leq r_1$. Let Assumptions \ref{bounds_parameter}, \ref{bounds_statistics}, \ref{bounds_statistics_maximum}, and \ref{lambdamin} be satisfied. Let 
	\begin{align}
	n & \geq O\bigg(\frac{ \dimOne^2  \dimTwo^2 }{\alpha^4 \lambdaMin^2}\log\Big(\frac{\dimOne \dimTwo}{\delta}\Big)\bigg).
	\end{align}
	Let $\eta = 1/\dimOne \dimTwo \dimThree \phiMax^2\exp(r_1d_1)$ and $\Theta^{(0)} = \boldsymbol{0}$. Then, Algorithm \ref{alg:GradientDescent} is guaranteed to produce an $\epsilon$-optimal solution $\hThetaEps$ such that $\|\hThetaEps - \ThetaStar\|_{\mathrm{T}} \leq \alpha$,
	with probability at least $1-\delta$ and with number of computations of the order
	\begin{align}
	O\bigg(\frac{ \dimOne^4  \dimTwo^4 }{\alpha^6 \lambdaMin^3}\log\Big(\frac{\dimOne \dimTwo}{\delta}\Big)\bigg).
	\end{align}
\end{restatable}
\begin{proof}[Proof of Corollary \ref{corr_low_rank_comp} ]
	The computational cost of projecting on the nuclear ball is $O(\dimOne \dimTwo \min\{\dimOne,  \dimTwo\})$ (see \cite{Jaggi2013} and note $\dimThree = O(1)$).
	The computational cost of the operation $\Theta_{(t)} - \eta  \nabla \cL_{n}(\Theta_{(t)}) - \Theta$ is $O(\dimOne \dimTwo n)$ because ($\dimThree = O(1)$).  Therefore, the computational cost of the step $\Theta_{(t+1)} \leftarrow \argmin_{\Theta \in \parameterSet} \| \Theta_{(t)} - \eta  \nabla \cL_{n}(\Theta_{(t)}) - \Theta\|$ of Algorithm \ref{alg:GradientDescent} is $O(\dimOne \dimTwo\max\{\min\{\dimOne,  \dimTwo\},n\})$.

	From Lemma \ref{lemma:gradient_descent}, Algorithm \ref{alg:GradientDescent} returns an $\epsilon$-optimal solution $\hThetaEps$ scales as
	\begin{align}
	\tau \geq \frac{2\dimOne \dimTwo\phiMax^2\exp(\br^T\bd)}{\epsilon} \| \hTheta_n\|^2_{\mathrm{F}}.
	\end{align}
	Also, $\| \hTheta_n\|^2_{\mathrm{F}} \leq \| \hTheta_n\|^2_{\star} \leq r_1^2.$ Combining everything, the computational cost is of the order $O\Big(\frac{ \dimOne^2  \dimTwo^2  \max\{\min\{\dimOne,  \dimTwo\},n\}}{\epsilon}\Big)$.
	Using Theorem \ref{thm:finite_sample}, and plugging in $n = O\Big(\frac{ \dimOne^2  \dimTwo^2 }{\alpha^4 \lambdaMin^2}\log\Big(\frac{\dimOne \dimTwo}{\delta}\Big)\Big)$ and $\epsilon = O(\alpha^2 \lambdaMin)$ completes the proof. 
\end{proof}

\subsection{Sparse-plus-low-rank decomposition}
\label{appendix:corr_sparse_low_rank_comp}
\begin{restatable}{corollary}{corrSparseLowRank}\label{corr_sparse_low_rank_comp}(Sparse-plus-low-rank decomposition)
	Suppose $\ThetaStar$ has a sparse-plus-low-rank decomposition i.e., $\ThetaStar = (\ThetaStarOne, \ThetaStarTwo)$ such that $\|\ThetaStarOne\|_{1,1} \leq r_1$ and $\|\ThetaStarTwo\|_{\star} \leq r_2$. Let Assumptions \ref{bounds_parameter}, \ref{bounds_statistics}, \ref{bounds_statistics_maximum}, and \ref{lambdamin} be satisfied. Let 
	\begin{align}
	n & \geq O\bigg(\frac{ \dimOne^2  \dimTwo^2 }{\alpha^4 \lambdaMin^2}\log\Big(\frac{\dimOne \dimTwo}{\delta}\Big)\bigg).
	\end{align}
	Let $\eta = 1/\dimOne \dimTwo \dimThree \phiMax^2\exp(r_1d_1 + r_2d_2)$ and $\Theta^{(0)} = \boldsymbol{0}$. Then, Algorithm \ref{alg:GradientDescent} is guaranteed to produce an $\epsilon$-optimal solution $\hThetaEps$ such that $\|\hThetaEps - \ThetaStar\|_{\mathrm{T}} \leq \alpha$,
	with probability at least $1-\delta$ and with number of computations of the order
	\begin{align}
	O\bigg(\frac{ \dimOne^4  \dimTwo^4 }{\alpha^6 \lambdaMin^3}\log\Big(\frac{\dimOne \dimTwo}{\delta}\Big)\bigg).
	\end{align}
\end{restatable}
\begin{proof}[Proof of Corollary \ref{corr_sparse_low_rank_comp} ]
	The proof follows directly from the proofs of Corollary \ref{corr_sparse_comp} and Corollary \ref{corr_low_rank_comp}.
\end{proof}

\section{Examples}
\label{appendix:examples}
In this Section, we provide a more elaborate discussion on the examples of natural parameters and statistics from Section \ref{subsec:examples}.
\subsection{Sparse-plus-low-rank decomposition}
The natural statistic $\Phi$ of an exponential family is such that for any $i_1 \neq i_2 \in [\dimOne], j_1 \neq j_2 \in [\dimTwo], l_1 \neq l_2 \in [\dimThree]$, $\Phi_{i_1j_1l_1} \neq \Phi_{i_2j_2l_2}$. Further, an exponential family is minimal if there does not exist a non-zero tensor $\bU \in \Reals^{\dimOne\times  \dimTwo \times  \dimThree} $ 
such that $\sum_{i \in [\dimOne], j \in [\dimTwo], l \in [\dimThree]} \bU_{ijl} \Phi_{ijl}(\svbx)$ is equal to a constant for all $\svbx \in \cX$. However, for the sparse-plus-low-rank decomposition, it is desirable to let $\PhiOne = \PhiTwo$ (see \cite{ChandrasekaranPW2010, MengEH2014}). In this scenario, there exists a non-zero tensor $\bU \in \Reals^{\dimOne\times  \dimTwo \times  \dimThree} $ 
such that $\sum_{i \in [\dimOne], j \in [\dimTwo], l \in [\dimThree]} \bU_{ijl} \Phi_{ijl}(\svbx) = 0$ for all $\svbx \in \cX$ for e.g., this is true if $\bU^{(1)} = -\bU^{(2)}$. In this situation, we say an exponential family is minimal if there does not exist a non-zero tensor $\bU \in \Reals^{\dimOne\times  \dimTwo \times  \dimThree} $ 
such that $\sum_{l \in [\dimThree]} \bU^{(l)} \neq 0$ as well as $\sum_{i \in [\dimOne], j \in [\dimTwo], l \in [\dimThree]} \bU_{ijl} \Phi_{ijl}(\svbx)$ is equal to a constant for all $\svbx \in \cX$. Therefore, it is often convenient to represent the tensor $\bU$ in terms of a matrix and define minimality of an exponential family in terms of this new matrix.

\subsection{Assumptions \ref{bounds_parameter} and \ref{bounds_statistics}}
While we expect the constants $\br$ in Assumption \ref{bounds_parameter} and $\bd$ in Assumption \ref{bounds_statistics} to be $O(1)$ for most applications, the sample complexity and the computational complexity in Theorem \ref{thm:finite_sample} would still be $O\Big(\mathrm{poly}\Big( \frac{\dimOne \dimTwo}{\alpha}\Big)\Big)$ as long as $\br$ and $\bd$ are $O\Big(\mathrm{log}(\dimOne \dimTwo)\Big)$.
 
\subsection{Polynomial natural statistic}
\label{appendix:poly}
Suppose the natural statistics are polynomials of $\rvbx$ with maximum degree $l$, i.e., $\prod_{i \in [p]} x_i^{l_i}$ such that $l_i \geq 0$ $\forall i \in [p]$ and $\sum_{i \in [p]} l_i \leq l$. 
\begin{itemize}[leftmargin=6mm, itemsep=-0.5mm]
	 \item Let $\cX = [0,b]$ for $b \in \Reals$. We will first show that $\phiMax = 2b^l$. We have
\begin{align}
\| \varPhi(\svbx) \|_{\max} & = \max_{u \in [\dimOne], v \in [ \dimTwo], w \in [\dimThree]} |\varPhi_{uvw}(\svbx)|\\
& \stackrel{(a)}{=} \max_{u \in [\dimOne], v \in [ \dimTwo], w \in [\dimThree]} \Big| \Phi_{uvw}(\svbx) - \Expectation_{\Uniform}[\Phi_{uvw}(\rvbx)] \Big|\\
& \stackrel{(b)}{\leq} \max_{u \in [\dimOne], v \in [ \dimTwo], w \in [\dimThree]}\Big| \Phi_{uvw}(\svbx) \Big| + \max_{u \in [\dimOne], v \in [ \dimTwo], w \in [\dimThree]} \Big| \Expectation_{\Uniform}[\Phi_{uvw}(\rvbx)] \Big|\\
& \leq 2 \max_{\svbx \in \cX} \max_{u \in [\dimOne], v \in [ \dimTwo], w \in [\dimThree]} \Big| \Phi_{uvw}(\svbx) \Big| \leq 2b^l.
\end{align}
where $(a)$ follows from Definition \ref{def:css} and $(b)$ follows from the triangle inequality.
	\item Suppose $\ThetaStar$ has a sparse decomposition i.e., $\ThetaStar = (\ThetaStarOne)$ and $\|\ThetaStarOne\|_{1,1} \leq r_1$. The dual norm of the matrix $L_{1,1}$ norm is the matrix maximum norm. Then, if $\cX = [0,b]$ for $b \in \Reals$,
	\begin{align}
	\cR^*_1(\varPhi^{(1)}(\svbx)) = \| \varPhi^{(1)}(\svbx) \|_{\max} = \| \varPhi(\svbx) \|_{\max} \leq \phiMax = 2b^l.
	\end{align}
	\item Suppose $\ThetaStar$ has a low-rank decomposition i.e., $\ThetaStar = (\ThetaStarOne)$ and $\|\ThetaStar\|_{\star} \leq r_1$. The dual norm of the matrix nuclear norm is the matrix spectral norm. Then, 
	\begin{align}
	\cR^*_1(\varPhi^{(1)}(\svbx)) = \| \varPhi^{(1)}(\svbx) \|.
	\end{align}
	Let $l = 2$, and $\cX = \cB(0,b)$. Observe that by writing $\varPhi^{(1)}(\svbx)  = \tilde{\rvx}\tilde{\rvx}^T$ where $\tilde{\rvx} = (1,\rvx_1,\cdots, \rvx_p)$, we have
	\begin{align}
	\| \varPhi^{(1)}(\svbx) \| \leq 2\Big(1 + \sum_{i \in [p]} \svbx_i^2\Big) \leq 2(1+b^2).  
	\end{align}
	\item Suppose $\ThetaStar$ has a sparse-plus-low-rank decomposition i.e., $\ThetaStar = (\ThetaStarOne, \ThetaStarTwo)$ such that $\|\ThetaStarOne\|_{1,1} \leq r_1$ and $\|\ThetaStarTwo\|_{\star} \leq r_2$. The dual norm of the matrix $L_{1,1}$ norm is the matrix maximum norm and the dual norm of the matrix nuclear norm is the matrix spectral norm. Let $l = 2$, and $\cX = \cB(0,b)$. Then,
	\begin{align}
	\bcR^*(\varPhi(\svbx)) \leq (\| \varPhi^{(1)}(\svbx) \|_{\max} , \| \varPhi^{(2)}(\svbx) \|)  \leq (2b^2, 2+2b^2).
	\end{align}
\end{itemize}
\subsection{Trigonometric natural statistic}
\label{appendix:sinu}
Suppose the natural statistics are sines and cosines of $\rvbx$ with $l$ different frequencies, i.e., $\sin(\sum_{i \in [p]}l_ix_i)$ $\cup$ $\cos(\sum_{i \in [p]}l_ix_i)$ such that $l_i \in [l] \cup \{0\}$.
\begin{itemize}[leftmargin=6mm, itemsep=-0.5mm]
	\item Let $\cX \subset \Reals^{p}$. We will first show that $\phiMax = 2$. We have
\begin{align}
\| \varPhi(\svbx) \|_{\max} & = \max_{u \in [\dimOne], v \in [ \dimTwo], w \in [\dimThree]} |\varPhi_{uvw}(\svbx)|\\
& \stackrel{(a)}{=} \max_{u \in [\dimOne], v \in [ \dimTwo], w \in [\dimThree]} \Big| \Phi_{uvw}(\svbx) - \Expectation_{\Uniform}[\Phi_{uvw}(\rvbx)] \Big|\\
& \stackrel{(b)}{\leq} \max_{u \in [\dimOne], v \in [ \dimTwo], w \in [\dimThree]}\Big| \Phi_{uvw}(\svbx) \Big| + \max_{u \in [\dimOne], v \in [ \dimTwo], w \in [\dimThree]} \Big| \Expectation_{\Uniform}[\Phi_{uvw}(\rvbx)] \Big|\\
& \leq 2 \max_{\svbx \in \cX} \max_{u \in [\dimOne], v \in [ \dimTwo], w \in [\dimThree]} \Big| \Phi_{uvw}(\svbx) \Big| \leq 2.
\end{align}
where $(a)$ follows from Definition \ref{def:css} and $(b)$ follows from the triangle inequality.
\item Suppose $\ThetaStar$ has a sparse decomposition i.e., $\ThetaStar = (\ThetaStarOne)$ and $\|\ThetaStarOne\|_{1,1} \leq r_1$. The dual norm of the matrix $L_{1,1}$ norm is the matrix maximum norm. Then, for any $\cX \subset \Reals^{p}$,
\begin{align}
\cR^*_1(\varPhi^{(1)}(\svbx)) = \| \varPhi^{(1)}(\svbx) \|_{\max} = \| \varPhi(\svbx) \|_{\max} \leq \phiMax = 2.
\end{align}
\end{itemize}
\subsection{Combinations of polynomial and trigonometric statistics}
Suppose the natural statistics are combinations of polynomials of $\rvbx$ with maximum degree $l$, i.e., $\prod_{i \in [p]} x_i^{l_i}$ such that $l_i \geq 0$ $\forall i \in [p]$ and $\sum_{i \in [p]} l_i \leq l$ as well as sines and cosines of $\rvbx$ with $\tilde{l}$ different frequencies, i.e., $\sin(\sum_{i \in [p]}l_ix_i)$ $\cup$ $\cos(\sum_{i \in [p]}l_ix_i)$ such that $l_i \in [\tilde{l}] \cup \{0\}$.
\begin{itemize}[leftmargin=6mm, itemsep=-0.5mm]
	\item Let $\cX = [0,b]$ for $b \in \Reals$. From Appendix \ref{appendix:poly} and Appendix \ref{appendix:sinu}, it is easy to verify that $\phiMax = \max\{2,2b^l\}$. 
	\item Suppose $\ThetaStar$ has a sparse decomposition i.e., $\ThetaStar = (\ThetaStarOne)$ and $\|\ThetaStarOne\|_{1,1} \leq r_1$. The dual norm of the matrix $L_{1,1}$ norm is the matrix maximum norm. Then, if $\cX = [0,b]$ for $b \in \Reals$, it is easy to verify that 
	\begin{align}
	\cR^*_1(\varPhi^{(1)}(\svbx)) = \| \varPhi^{(1)}(\svbx) \|_{\max} = \| \varPhi(\svbx) \|_{\max} \leq \phiMax =  \max\{2,2b^l\}.
	\end{align}
\end{itemize}

\section{Property \ref{property:norms} for norms of interest}
\label{appendix:dual_norm}
In this Section, we show that the $g$ defined in Property \ref{property:norms} in Section \ref{sec:main results} is 1 for the entry-wise $L_{p,q}$ norm $(p,q \geq 1)$, the Schatten $p$-norm $(p \geq 1)$, and the operator $p$-norm $(p \geq 1)$.
\subsection{The entry-wise $L_{p,q}$ norm}
Let $\tilde{\cR}(\cdot)$ denote the entry-wise $L_{p,q}$ norm for some $p,q \geq 1$. We will show that for any matrix $\bM \in \Reals^{\dimOne\times \dimTwo}$
\begin{align}
\tilde{\cR}(\bM) \leq \|\bM\|_{\max} \times \dimOne^{\frac{1}{p}}  \dimTwo^{\frac{1}{q}}.
\end{align}
 By the definition of the entry-wise $L_{p,q}$ norm, we have
\begin{align}
\tilde{\cR}(\bM) = \bigg(\sum_{j \in  [\dimTwo]} \bigg(\sum_{i \in [\dimOne]} | M_{ij} |^p \bigg)^{\frac{q}{p}} \bigg)^{\frac{1}{q}} & \leq \bigg(\sum_{j \in  [\dimTwo]} \bigg(\sum_{i \in [\dimOne]} \|\bM\|_{\max}^p \bigg)^{\frac{q}{p}} \bigg)^{\frac{1}{q}} \\
& = \dimOne^{\frac{1}{p}}  \dimTwo^{\frac{1}{q}} \|\bM\|_{\max} \leq \dimOne \dimTwo \|\bM\|_{\max}.
\end{align}

\subsection{The Schatten $p$-norm}
Let $\tilde{\cR}(\cdot)$ denote the Schatten $p$-norm for some $p \geq 1$. We will show that for any matrix $\bM \in \Reals^{\dimOne\times \dimTwo}$
\begin{align}
\tilde{\cR}(\bM) \leq \|\bM\|_{\max} \times\sqrt{\min\{\dimOne, \dimTwo\}\dimOne \dimTwo} .
\end{align}
Let the rank of $\bM$ be denoted by $r$ and the singular values of $\bM$ be denoted by $\sigma_i(\bM)$ for $i \in [r]$.
By the definition of the Schatten $p$-norm, we have
\begin{align}
\tilde{\cR}(\bM) = \bigg(\sum_{i \in [r]} \sigma_i^p(\bM) \bigg)^{\frac{1}{p}} \stackrel{(a)}{\leq} \sum_{i \in [r]} \sigma_i(\bM) & \stackrel{(b)}{\leq} \sqrt{r\dimOne \dimTwo} \|\bM\|_{\max} \\
& \stackrel{(c)}{\leq} \sqrt{\min\{\dimOne, \dimTwo\}\dimOne \dimTwo} \|\bM\|_{\max} \leq \dimOne \dimTwo \|\bM\|_{\max}
\end{align}
where $(a)$ follows because of the monotonicity of the Schatten $p$-norms, $(b)$ follows because $\|\bM\|_{\star} \leq \sqrt{r\dimOne \dimTwo} \|\bM\|_{\max}$, and $(c)$ follows because $r \leq \min\{\dimOne, \dimTwo\}$.

\subsection{The operator $p$-norm }
Let $\tilde{\cR}(\cdot)$ denote the operator $p$-norm for some $p \geq 1$. We will show that for any matrix $\bM \in \Reals^{\dimOne\times \dimTwo}$
\begin{align}
\tilde{\cR}(\bM) \leq \|\bM\|_{\max} \times \dimOne^{\frac{1}{p}}  \dimTwo^{1-\frac{1}{p}}.
\end{align}
Let $q = \frac{p}{p-1}$. For $i \in \dimOne$, let $[\bM]_{i}$ denote the $i^{th}$ row of $\bM$. By the definition of the operator $p$-norm, we have
\begin{align}
\tilde{\cR}(\bM) = \max_{\svby : \| \svby\|_p = 1} \|\bM \svby \|_{p} & \stackrel{(a)}{\leq} \dimOne^{\frac{1}{p}} \max_{\svby : \| \svby\|_p = 1}  \|\bM \svby \|_{\infty} \\
& \stackrel{(b)}{\leq} \dimOne^{\frac{1}{p}} \max_{\svby : \| \svby\|_p = 1}  \max_{i \in [\dimOne]} \|[\bM]_{i}\|_{q} \|\svby\|_{p} \\
& \leq \dimOne^{\frac{1}{p}}   \max_{i \in [\dimOne]} \|[\bM]_{i}\|_{q} \\
& \stackrel{(c)}{\leq} \dimOne^{\frac{1}{p}}  \dimTwo^{\frac{1}{q}}  \max_{i \in [\dimOne]} \|[\bM]_{i}\|_{\infty} \\
& = \dimOne^{\frac{1}{p}}  \dimTwo^{1-\frac{1}{p}} \|\bM\|_{\max} \leq \dimOne \dimTwo \|\bM\|_{\max}
\end{align}
where $(a)$ follows because $\|\svbv\|_p \leq m^{\frac{1}{p}} \|\svbv\|_{\infty}$ for any vector $\svbv \in \Reals^{m}$ and $p \geq  1$, $(b)$ follows from the definition of the infinity norm of a vector and using the H\"{o}lder's inequality, and $(c)$ follows because $\|\svbv\|_q \leq m^{\frac{1}{q}} \|\svbv\|_{\infty}$ for any vector $\svbv \in \Reals^{m}$ and $q \geq  1$.

\end{document}